\documentclass{article}

\usepackage[T1]{fontenc}
\usepackage{microtype}
\usepackage{subfigure}
\usepackage{booktabs} 
\usepackage{amsfonts, amsmath, amssymb, amsthm}
\usepackage{xfrac}
\usepackage{graphicx, xcolor, colortbl}
\usepackage{footnote}
\usepackage{tablefootnote}
\usepackage{natbib}
\usepackage{dsfont}
\usepackage{bbm}
\usepackage{algorithm}
\usepackage{algorithmic}
\usepackage{import}
\usepackage{mathtools}
\usepackage{bm}
\usepackage{xspace}
\usepackage{thmtools,thm-restate} 
\usepackage{accents}
\usepackage{framed}
\usepackage{enumitem}
\usepackage{booktabs}
\usepackage{color, colortbl}

\usepackage[final]{sty/neurips_2022}

\usepackage{notation/notations}

\usepackage{minitoc}

\theoremstyle{plain}
\newtheorem{theorem}{Theorem}[section]
\newtheorem{proposition}[theorem]{Proposition}
\newtheorem{lemma}[theorem]{Lemma}
\newtheorem{corollary}[theorem]{Corollary}
\theoremstyle{definition}

\theoremstyle{remark}
\newtheorem{remark}[theorem]{Remark}

\usepackage{scrextend}

\usepackage{multirow}

\usepackage{todonotes}

\usepackage[utf8]{inputenc} 
\usepackage[T1]{fontenc}    
\usepackage{hyperref}       
\usepackage{url}            
\usepackage{booktabs}       
\usepackage{amsfonts}       
\usepackage{nicefrac}       
\usepackage{microtype}      
\usepackage{xcolor}         

\title{Optimistic Posterior Sampling for Reinforcement
Learning with Few Samples and Tight Guarantees}

\author{%
  Daniil Tiapkin \\
  HSE University\\
  \texttt{dtyapkin@hse.ru} \\
  \And
  Denis Belomestny \\
  Duisburg-Essen University, HSE University \\
  \texttt{denis.belomestny@uni-due.de} \\
  \And
  Daniele Calandriello \\
  DeepMind \\
  \texttt{dcalandriello@deepmind.com} \\
  \And
  \' Eric Moulines \\
  \' Ecole Polytechnique \\
  \texttt{eric.moulines@polytechnique.edu} \\
  \And
  Remi Munos \\
  DeepMind \\
  \texttt{munos@deepmind.com} \\
  \And
  Alexey Naumov \\
  HSE University \\
  \texttt{anaumov@hse.ru} \\
  \AND 
  Mark Rowland \\
  DeepMind \\
  \texttt{markrowland@deepmind.com} \\
  \And 
  Michal Valko \\
  DeepMind \\
  \texttt{valkom@deepmind.com} \\
  \And 
  Pierre M\' enard \\
  ENS Lyon\\
  \texttt{pierre.menard@ens-lyon.fr} \\
}

\begin{document}

\maketitle

\begin{abstract}
We consider reinforcement learning in an environment modeled by an episodic, finite, stage-dependent Markov decision process of horizon $H$ with $S$ states, and $A$ actions. The performance of an agent is measured by the regret after interacting with the environment for~$T$ episodes. We propose an optimistic posterior sampling algorithm for reinforcement learning (\OPSRL), a simple variant of posterior sampling that only needs a number of posterior samples logarithmic  in $H$, $S$, $A$, and~$T$ per state-action pair. For \OPSRL we guarantee a high-probability regret bound of order at most $\tcO(\sqrt{H^3SAT})$ ignoring
$\text{poly}\log(HSAT)$ terms. The key novel technical ingredient is a new sharp anti-concentration inequality for linear forms which may  be of independent interest. Specifically, we extend the normal approximation-based lower bound for Beta distributions by \citet{alfers1984normal} to Dirichlet distributions. Our bound matches the lower bound of order $\Omega(\sqrt{H^3SAT})$, thereby answering the open problems raised by \citet{agrawal2020posterior} for the episodic setting. 

\end{abstract}

\doparttoc 
\faketableofcontents 

\section{Introduction}
\label{sec:introduction}

In reinforcement learning an agent interacts with an environment, whose underlying mechanism is unknown, by sequentially taking actions, receiving rewards, and transitioning to the next state \citep{SuttonBarto98}. With the goal of maximizing the expected sum of the collected rewards, the agent must carefully balance between \emph{exploring} in order to gather more information about the environment and \emph{exploiting} the current knowledge to collect the rewards. In this paper, we are interested in solving this exploration-exploitation dilemma by injecting noise into the agent's decision-making process.

We model the environment as an episodic, finite, unknown Markov decision process (MDP) of horizon $H,$ with $S$ states and $A$ actions. In particular, we consider the \emph{stage-dependent} setting where the rewards and the transition probability distributions can vary within an episode. After $T$ episodes, the performance of an agent is measured through \emph{regret} which is the difference between the cumulative reward the agent could have obtained by acting optimally and what the agent really obtained.

\citet{jin2018is} and \citet{domingues2020regret} provide a problem-independent lower bound of order $\Omega(\sqrt{H^3SAT})$ for this setting; see also \citet{azar2017minimax} for a lower bound when the transitions are stage-independent.

One generic solution to the exploration-exploitation dilemma is the \emph{principle of optimism in the face of uncertainty}. A simple way to implement this principle consists in building \emph{upper confidence bound (UCB)} on the optimal Q-value function through the addition of \emph{bonuses} to the rewards. 
This is done by either model-based algorithms \citep{azar2017minimax,dann2017unifying,Zanette19Euler} or model-free algorithms \citep{jin2018is,zhang2020advantage,menard2021ucb}; see also \citep{jaksch2010near,fruit2018efficient,talebi2018variance} for the non-episodic setting. Notably, among others, both the upper confidence bound value iteration (\UCBVI) of \citet{azar2017minimax} and the UCB-Advantage algorithm of~\citet{zhang2020advantage} enjoys a problem-independent regret bound\footnote{We translate all the bounds to the \emph{stage-dependent} setting by multiplying the regret bounds in the stage-independent setting by $\sqrt{H}$, see \citet{jin2018is}.} of order\footnote{In the $\tcO(\cdot)$ notation we ignore terms poly-$\log$  in $H,S,A,T$.}
$\tcO(\sqrt{H^3SAT})$ that matches the aforementioned lower bound for $T$ large enough and up to  terms poly-logarithmic in $H,S,A,T$.

Another way is to implement the optimism by \emph{injecting noise}. A typical example is the random least-square value iteration (\RLSVI, \citealp{osband16generalization,russo2019worst}) algorithm which at each episode computes new Q-values by noisy value iteration from an estimated model  and then acts greedily with respect to them. In particular, a Gaussian noise is added to the reward before applying the Bellman operator to encourage exploration. Indeed, when the variance of the noise is carefully chosen, it allows to obtain optimistic Q-values with at least a fixed probability. \citet{russo2019worst} first proved a regret bound of order $\tcO(H^2S^{3/2}\sqrt{AT})$ for \RLSVI.  Later, \citet{xiong2021nearoptimal} obtained an optimal regret bound of order $\tcO(\sqrt{H^3SAT})$ for a modified version of \RLSVI where the variance of the injected Gaussian noise is scaled by a term similar to the Bernstein bonuses used in \UCBVI. Note that the \RLSVI was also successfully extended beyond the tabular case to settings with function approximation, e.g. see \citealp{ishfaq2021randomized, zanette2020frequentist}.

Recently, \citet{pacchiano2021towards} analyzed a version of \RLSVI where the Gaussian noise is replaced by a bootstrap sample of \emph{the past rewards} and added pseudo rewards in the same fashion as \citet{kveton2019garbage}. The algorithm proposed by \citet{pacchiano2021towards}, comes with a regret bound of order $\tcO(H^2S\sqrt{AT})$.

By generalizing the Thompson sampling algorithm \citep{thompson1933on} originally given for stochastic multi-armed bandit, \citet{obsband2013more} propose a posterior sampling for reinforcement learning (\PSRL).  \PSRL algorithm also relies on noise to drive exploration. The general idea behind it is to maintain a \emph{surrogate Bayesian model} on the MDP, for instance, a Dirichlet posterior on the transition probability distribution if the rewards are known. At each episode, a new MDP is sampled (i.e., a transition probability for each state-action pair) according to the posterior distribution of the Bayesian model. Then, the agent acts optimally in this sampled MDP. As the posterior is not well concentrated in the unexplored region of the MDP, the probability that the Q-value of the sampled MDP is optimistic in this region is high. Therefore, the agent will be incentivized to explore. Although the original Thompson sampling is well-studied in the frequentist setting \citep{agrawal2012analysis,kaufmann2012thompson,agrawal2013further,zhang2022feel} and the Bayesian setting \citep{thompson1933on,russo2016information,russo2014learning}, most of the analysis of \PSRL only provide Bayesian regret bounds \citep{obsband2013more,abbasi2015bayesian,osband16generalization,ouyang2017learning,osband2017why}, i.e., when the true MDP is effectively sampled according to the prior of the surrogate Bayesian model. Despite this lack of guarantees, \PSRL demonstrates competitive empirical performance in comparison to bonus-based algorithms \citep{obsband2013more,osband2017why}. Additionally, the exploration mechanism used by \PSRL (and \RLSVI) was successfully extended outside the tabular setting and used in deep RL environments \citep{osband2016deep,osband18randomized,osband2019deep}.

One exception to the above is the work of~\citet{agrawal2020posterior} that studies \PSRL from a \textit{frequentist} perspective in the infinite-horizon, non-episodic average reward setting. In particular, they provide a regret bound\footnote{As acknowledged by the authors, there was a mistake in the initial submission of their work where the previously announced bound was claimed to be $\sqrt{S}$ better, see \citet{agrawal17correct,qian2020concentration} } of order\footnote{We translate all the bounds from the infinite-horizon, non-episodic average reward setting to our setting by identifying the diameter with the horizon $H$ and multiplying the bound by $\sqrt{H}$ because of our stage-dependent transitions assumption.} $\tcO(H^2S\sqrt{AT})$ 
for an optimistic version of \PSRL that we call \SOPSRL  since it switches between two types of sampling of the transitions: (1) \textit{simple optimistic sampling}, when the number of observed transitions at a given state-action pair is too small. In this case, the sampled transition is a random mixture between the uniform distribution over the states and an empirical estimate of the true transition biased by some bonus-like terms; or if the number of observed transitions at a given state-action pair is large enough (2) \textit{optimistic posterior sampling,}  where $\tcO(S)$ samples from an inflated Dirichlet posterior are used instead of one sample used in \PSRL. Then, from these $\tcO(S)$ sampled transition probabilities we select the most optimistic one i.e., the one leading to the largest optimal Q-value.

The key idea underpinning the analysis of \SOPSRL, and \PSRL-like algorithms in general, is to control the deviations of the Dirichlet posterior on the transition probability distributions. In particular, we need to show that the \textit{posterior spreads enough to ensure optimism}. To this end, \citet{agrawal2020posterior} derive an anti-concentration bound for any fixed projection of a Dirichlet random vector. The latter result in turn relies upon an equivalent representation of a Dirichlet vector in terms of independent Beta random variables and an anti-concentration bound for the corresponding Beta distribution. However, this anti-concentration inequality is not uniformly tight, in particular its polynomial dependence on the number of states $S$ is suboptimal.

\citet{agrawal2020posterior} conclude with two open problems. The first question is whether one can reduce the number of posterior samples required per state-action pair from $\tcO(S)$ to constant or logarithmic in $S$. The second asks if it is possible to obtain a near-optimal regret bound and in particular to improve the dependence on $S$. In this paper, we \emph{answer both of them in the affirmative} in the episodic setting. Indeed, we propose optimistic posterior sampling algorithm  for reinforcement learning (\OPSRL)  that only requires $\tcO(1)$ samples from an inflated posterior while enjoying a near-optimal problem independent regret bound of order $\tcO(\sqrt{H^3 SAT})$. \OPSRL is a simple optimistic variant of \PSRL which, in particular, does not rely at all on ''simple'' (bonus-based) optimistic sampling.

The essential ingredient for \OPSRL's analysis is our \textit{novel anti-concentration bound for the projections of a Dirichlet random vector} (Theorem~\ref{thm:gaussian_lb}). We base it on a tight Gaussian approximation for linear forms of a Dirichlet random vector. This latter approximation can be seen as a substantial generalization to Dirichlet distributions of the result obtained by \cite{alfers1984normal} for the case of Beta distributions. We obtain this approximation through a refined non-asymptotic analysis of the integral representation for the density of a linear form of a Dirichlet random vector, which was first derived\footnote{Note that the anti-concentration inequality proved by \citet{tiapkin22dirichlet} based on the same integral representation is insufficient for our needs, see Remark~\ref{rem:bound_icml_isufficent} for a discussion.} by \citet{tiapkin22dirichlet}. We believe that the new anti-concentration inequality presented in this work could be of independent interest, e.g., to tighten or simplify analysis of non-parametric Thompson sampling like algorithms \citep{riou20a,baudry21optimal,baudry2021optimality} for stochastic multi-armed bandits.

\begin{itemize}[leftmargin=0.5cm]
    \item We propose the \OPSRL algorithm for tabular, stage-dependent, episodic RL. It is a simple optimistic variant of the \PSRL algorithm that only needs $\tcO(1)$ posterior samples per state-action pair. For \OPSRL, we provide a regret bound of order $\tcO(\sqrt{H^3SAT})$ matching the problem independent lower bound up to poly-$\log$ terms. In particular we answer positively to two open questions by \citet{agrawal2020posterior} in the episodic setting.
    \item We derive a new anti-concentration inequality for a linear form of a Dirichlet random vector (Theorem~\ref{thm:gaussian_lb}) which is essential for the analysis of \OPSRL. This result is a generalization to the Dirichlet case of the one provided by \cite{alfers1984normal} for Beta distributions.
 \end{itemize}
\section{Setting}
\label{sec:setting}

 We consider a finite episodic MDP $\left(\cS, \cA, H, \{p_h\}_{h\in[H]},\{r_h\}_{h\in[H]}\right)$,  where $\cS$ is the set of states, $\cA$ is the set of actions, $H$ is the number of steps in one episode, $p_h(s'|s,a)$ is the probability transition from state~$s$ to state~$s'$ by taking the action $a$ at step $h,$ and $r_h(s,a)\in[0,1]$ is the bounded deterministic\footnote{We study deterministic rewards to simplify the proofs but our result extend to bounded random rewards as well.} reward received after taking the action $a$ in state $s$ at step $h$. Note that we consider the general case of rewards and transition functions that are possibly non-stationary, i.e., that are allowed to depend on the decision step $h$ in the episode. We denote by $S$ and $A$ the number of states and actions, respectively.

\paragraph{Policy \& value functions} A \emph{deterministic} policy $\pi$ is a collection of functions $\pi_h : \cS \to \cA$ for all $h\in [H]$, where every $\pi_h$  maps each state to a \emph{single} action. The value functions of $\pi$, denoted by $V_h^\pi$, as well as the optimal value functions, denoted by $\Vstar_h$ are given by the Bellman and the optimal Bellman equations,
\begin{small}
\begin{align*}
	Q_h^{\pi}(s,a) &= r_h(s,a) + p_h V_{h+1}^\pi(s,a) & V_h^\pi(s) &= \pi_h Q_h^\pi (s)\\
  Q_h^\star(s,a) &=  r_h(s,a) + p_h V_{h+1}^\star(s,a) & V_h^\star(s) &= \max_a Q_h^\star (s, a),
\end{align*}
\end{small}%
\!where by definition, $V_{H+1}^\star \triangleq V_{H+1}^\pi \triangleq 0$. Furthermore, $p_{h} f(s, a) \triangleq \E_{s' \sim p_h(\cdot | s, a)} \left[f(s')\right]$   denotes the expectation operator with respect to the transition probabilities $p_h$ and
$\pi_h g(s) \triangleq  g(s,\pi_h(s))$ denotes the composition with the policy~$\pi$ at step $h$.

\paragraph{Learning problem} The agent, to which the transitions are \emph{unknown} (the rewards are assumed to be known for simplicity), interacts with the environment during $T$ episodes of length $H$, with a \emph{fixed} initial state $s_1$.\footnote{As explained by \citet{fiechter1994efficient} and \citet{kaufmann2020adaptive}, if the first state is sampled randomly as $s_1\sim p,$ we can simply add an artificial first state $s_{1'}$ such that for  any action $a$, the transition probability is defined as the distribution $p_{1'}(s_{1'},a) \triangleq p.$} Before each episode $t$ the agent selects a policy $\pi^t$ based only on the past observed transitions up to episode $t-1$. At each step $h\in[H]$ in episode $t$, the agent observes a state $s_h^t\in\cS$, takes an action $\pi_h^t(s_h^t) = a_h^t\in\cA$ and  makes a transition to a new state $s_{h+1}^t$ according to the probability distribution $p_h(s_h^t,a_h^t)$ and receives a deterministic reward $r_h(s_h^t,a_h^t)$.

\paragraph{Regret} The quality of an agent is measured through its regret, that is the difference between what it could obtain (in expectation) by acting optimally and what it really gets,
\[
\regret^T \triangleq  \sum_{t=1}^T \Vstar_1(s_1)- V_1^{\pi^t}(s_1)\,.
\]

\paragraph{Counts} The number of times the state action-pair $(s,a)$ was visited in step $h$ in the first~$t$ episodes is denoted as $n_h^{t}(s,a) \triangleq  \sum_{i=1}^{t} \ind{\left\{(s_h^i,a_h^i) = (s,a)\right\}}$. Next, we define $n_h^{t}(s'|s,a) \triangleq \sum_{i=1}^{t} \ind{\big\{(s_h^i,a_h^i,s_{h+1}^i) = (s,a,s')\big\}}$ the number of transitions from $s$ to $s'$ at step $h$.

\paragraph{Improper Dirichlet distribution} For $m\in\N^*,$ the probability simplex of dimension $m$ is denoted by $\simplex_{m}$. For $\alpha \in (\R_{++})^{m+1},$ we denote by $\Dir(\alpha)$ the Dirichlet distribution on $\simplex_{m}$ with parameter~$\alpha$. We also extend this distribution to improper parameter $\alpha \in (\R_{+})^{m+1}$ such that $\sum_{i=0}^{m}  \alpha_i > 0$
 by injecting $\Dir((\alpha_i)_{i:\alpha_i>0})$ into $\simplex_{m}$. 
 Precisely, we say that $p \sim \Dir(\alpha)$ if $(p_i)_{i:\alpha_i>0} \sim \Dir((\alpha_i)_{i:\alpha_i>0})$ and all other coordinates are zero. 

\paragraph{Additional notation} For $N\in\N_{++},$ we define the set $[N]\triangleq \{1,\ldots,N\}$. We denote the uniform distribution over this set by $\Unif[N]$. The vector of dimension $N$ with all entries one is $\bOne^N \triangleq  (1,\ldots,1)^\top$.
 The empirical probability distribution $\hp^{\,t}_h(s,a)$ is defined as $\hp^{\,t}_h(s'|s,a) = n^{\,t}_h(s'|s,a) / n^{\,t}_h(s,a)$ if $n^t_h(s, a) > 0 $ and $\hp^{\,0}_h(s'|s,a)=1/S$ otherwise. Appendix~\ref{app:notations} references all the notation used.
\section{Algorithm}
\label{sec:algorithm}

In this section we describe the \OPSRL algorithm. In spirit, \OPSRL proceeds similarly as \PSRL except that it uses several posterior samples instead and acts optimistically with respect to them, explaining the name \textit{Optimistic Posterior Sampling for Reinforcement Learning} (\OPSRL).

\paragraph{Optimistic pseudo-state} In order to define the prior used by \OPSRL, we extend the state space $\cS$ by an absorbing pseudo-state $s_0$ with reward $r_h(s_0,a) \triangleq  \ur > 1$ for all $h,a$ and transition probability distribution $p_h(s'| s_0,a) \triangleq  \ind{\{s'=s_0\}}$. A similar pseudo-state was already introduce in previous works, see for example \citet{Brafman02RMAX,szita2008}. We denote by $\cS' = \cS\cup\{s_0\}$ the augmented states space and by $\Delta_{\cS'}$ the set of probability distributions over $\cS'$.

\paragraph{Pseudo-counts} We define the pseudo-counts, $\upn_h^t(s,a) \triangleq n_h^t(s,a)+n_0,$ as the counts shifted by an initial value $n_0$. This shift corresponds to prior transitions to the pseudo-state, that is $\upn_h^t(s'|s,a) \triangleq n_h^t(s'|s,a) + n_0 \ind\{s'=s_0\}$. Similar to the empirical transitions, we define a pseudo-empirical transition probability distribution as $\up^{\,t}_h(s,a) = \upn^{\,t}_h(s'|s,a) / \upn^{\,t}_h(s,a)$.

\paragraph{Inflated Bayesian model} Like \PSRL, we define a Bayesian model on the transition probability distributions, except that the prior/posterior is inflated. The practice of inflating the posterior is common in the analysis of Thompson sampling like algorithm, see \citet{agrawal2020posterior, abeille17linear}. Precisely, the inflated prior is a Dirichlet distribution $\Dir\!\Big(\big(\upn_h^0(s'|s,a)/\kappa\big)_{s'\in\cS'}\Big)$ parameterized by the initial pseudo-counts, and some constant $\kappa>0$ controlling the inflation. Thus the prior is a Dirac distribution at a deterministic transition leading to the artificial state $s_0$.
Then the inflated posterior is also a Dirichlet distribution $\Dir\!\Big(\big(\upn_h^t(s'|s,a)/\kappa\big)_{s'\in\cS'}\Big)$.  Note that the prior is a proper prior (i.e., a valid probability distribution), but it will be updated in an improper way, i.e., probability transitions with no mass under the prior could get mass in the posterior, as they get positive counts.

\paragraph{Optimistic posterior sampling} After episode $t,$ for each state-action pair $(s,a)$ and step $h\in[H]$ we sample $J$ independent transition probability distributions $\tp_h^{\,t,j}(s,a)\sim \Dir\!\Big(\big(\upn_h^t(s'|s,a)/\kappa\big)_{s'\in\cS'}\Big)$ from the inflated posterior. Then, the Q-values are obtained by optimistic backward induction with these transitions. Precisely the value after the last step is zero $\uV_{H+1}^t(s) \triangleq 0$ and the optimal Bellman equations become
\begin{align}
    \label{eq:optimistic_Bellman}
    \begin{split}
        \uQ_h^t(s,a) &\triangleq r_h(s,a)+\max_{j\in[J]} \tp_h^{\,t,j} \uV_{h+1}^t(s,a)\,,\\
        \uV_h^t(s) &\triangleq \max_{a\in\cA} \uQ_h^t(s, a)\,.
    \end{split}
\end{align}
The next policy is greedy with the Q-values $\pi_h^{t+1}(s) \in \argmax_{a\in\cA} \uQ_h^t(s,a)$. The complete procedure of \OPSRL is described in Algorithm~\ref{alg:OPSRL} for a general family of distributions parameterized by the pseudo-counts over the transitions instead of the inflated Dirichlet prior/posterior.

\begin{algorithm}[h!]
\centering
\caption{\OPSRL}
\label{alg:OPSRL}
\begin{algorithmic}[1]
  \STATE {\bfseries Input:} Family of probability distributions $\rho: \N_{+}^{S+1} \to \Delta_{\cS'}$ over transitions, initial pseudo-count~$\upn_h^0$, number of posterior samples $J$.
      \FOR{$t \in[T]$}
      \STATE For all $(s,a,h)\in \cS\times\cA\times[H],$   sample $J$ independent transitions 
      \[
      \tp_h^{\,t-1,j}(s,a)\sim \rho\big(\upn_h^{t-1}(s'|s,a)_{s'\in\cS'}\big),\quad j\in[J].
      \]
      \STATE Optimistic backward induction: set $\uV_{H+1}^{t-1}(s)= 0$ and recursively for $h\in [H],$ compute 
      \begin{align*}
        \uQ_h^{t-1}(s,a) &= r_h(s,a)+\max_{j\in[J]} \bigl\{\tp_h^{\,t-1,j} \uV_{h+1}^{t-1}(s,a)\bigr\}\,,\\
        \uV_h^{t-1}(s) &= \max_{a\in\cA} \uQ_h^{t-1}(s, a)\,,\\
        \pi_h^t(s) &\in \argmax_{a\in\cA} \uQ_{h}^{t-1} (s,a)\,.
      \end{align*}
      \FOR{$h \in [H]$}
        \STATE Play $a_h^t =  \pi_h^t(s_h^t) $.
        \STATE Observe $s_{h+1}^t\sim p_h(s_h^t,a_h^t)$.
        \STATE Increment the pseudo-count $\upn_h^t(s_{h+1}^t | s^t_h, a_h^t)$.
      \ENDFOR
  \ENDFOR
\end{algorithmic}
\end{algorithm}

\subsection{Analysis}
We fix $\delta\in(0,1)$ and the number of samples 
\begin{equation*}
\label{eq:def_J}
    J \triangleq \lceil c_J \cdot \log(2SAHT/\delta) \rceil,
\end{equation*}
where $c_J = 1/\log(2/(1 + \Phi(1)))$ and $\Phi(\cdot)$ is the  cumulative distribution function (CDF) of a normal distribution. Note that $J$ has a logarithmic dependence on $S,A,H,T,$ and $1/\delta$.

We now state the regret bound of \OPSRL with a full proof in Appendix~\ref{app:regret_bound_proof}.
and a sketch in Section~\ref{sec:proof_sketch}.
\begin{theorem}
\label{th:regret_bound_OPSRL} 
Consider a parameter $\delta \in (0,1)$. Let $\kappa \triangleq 2(\log(12 SAH/\delta) + 3\log(\rme\pi(2T+1)))$, $n_0 \triangleq \lceil \kappa(c_{0} + \log_{17/16}(T)) \rceil$, $\ur \triangleq 2$, where  $c_{0}$ is an absolute constant defined in \eqref{eq:constant_c0}; see Appendix~\ref{app:optimism}. Then for \OPSRL, with probability at least $1-\delta$, 
\[
    \regret^T = \cO\left( \sqrt{H^3 SAT L^3}  + H^3 S^2 A L^3 \right),
\]
 where $L \triangleq \cO(\log(HSAT/\delta))$.
\end{theorem}

\paragraph{Computational complexity} \OPSRL is a model-based algorithm, and thus gets the $\cO(HS^2A)$ space complexity as \PSRL. Since we need $\tcO(1)$ posterior samples per state-action pair the time complexity of \OPSRL is of order $\tcO(HS^2A)$ per episode, the same as \PSRL up to poly-logarithmic terms. Building on the idea of \citet{efroni2019tight}, in Appendix~\ref{app:lazy_OPSRL} we propose the \lazyOPSRL algorithm a more time-efficient version of \OPSRL. Instead of recomputing the Q-value by backward induction before each episode, \lazyOPSRL only performs one step of optimistic incremental planning at the visited states. It enjoys a regret bound of the same order $\tcO(\sqrt{H^3SAT})$ as \OPSRL but with an improved time-complexity per episode of $\cO(HSA)$, see Theorem~\ref{th:regret_bound_lazyOPSRL} in Appendix~\ref{app:lazy_OPSRL}.

\paragraph{Comparison with \SOPSRL and \PSRL} One structural difference between \OPSRL and \SOPSRL of \citet{agrawal17correct} is that \OPSRL only relies on optimistic posterior sampling while \SOPSRL also uses simple optimistic sampling: a  mixture of the uniform distribution over the states and an empirical estimate of the true transition kernel biased by some bonus-like terms. In particular, \OPSRL does not use bonus-like quantities which could lead to poor empirical performance \citep{osband2017why}. Another important issue is  the number of posterior samples. \SOPSRL needs $\tcO(S)$ posterior samples in order to obtain a regret bound of order $\tcO(H^{2}S\sqrt{AT})$ whereas \OPSRL needs \textit{only} $\tcO(1)$ samples \textit{and} obtains a better regret bound. Note that if we choose the number of posterior samples as $J=1$ in \OPSRL we recover \PSRL up to two technical differences: First, the posterior is inflated in order to increase its variance. This technical trick was already used by \citet{agrawal17correct} and allows to guarantee optimism with a small number of posterior samples, see Section~\ref{sec:proof_sketch}. Second, \OPSRL uses a particular prior which is a Dirac distribution at a deterministic transition towards an optimistic pseudo-state. This prior is needed to control the deviations of the (inflated) posterior, see Theorem~\ref{thm:lower_bound_dbc}.

\paragraph{Comparison with \RLSVI} Both \OPSRL and \RLSVI build on the same mechanism for exploration. \RLSVI just adds an Gaussian noise to the Q-values whereas \OPSRL injects the noise naturally via a random transition sampled from a Dirichlet distribution. As controlling  the deviation of the Q-value obtained with additive Gaussian noise is not difficult, the analysis of \RLSVI is relatively straightforward \citep{russo2019worst,ishfaq2021randomized}. On the contrary the analysis of \OPSRL is much more involved, see Section~\ref{sec:proof_sketch}. However, the benefit of optimistic posterior sampling in \OPSRL is that it adapts \emph{automatically} to the variance of the estimates of the transitions which is central for a regret bound with an optimal dependence on the horizon $H$ \citep{azar2017minimax}. Adapting to the variance with \RLSVI is much more involved and artificial, see \citet{xiong2021nearoptimal}. This is probably one reason why \RLSVI performs  empirically worse than \PSRL \citep{osband2016deep}.

\subsection{Proof sketch}
\label{sec:proof_sketch}
The proof of Theorem~\ref{th:regret_bound_OPSRL} consists of three important steps. The first step is devoted to the approximation for tails of weighted sums of Dirichlet distribution and embodies the main technical contribution of the paper.

\paragraph{Step 1. Exponential and Gaussian approximation for Dirichlet distribution}

The first result  generalizes  \cite{riou20a} to  Dirichlet distributions with real parameters.  Let us first recall the definition of the minimum Kullback-Leibler divergence for $p\in\simplex_{m}$ where $m\in\N^+$, a function $f:\{0,\ldots,m\}\to[0,b]$ for some $b\in\R^+$ and $u\in\R$,
 \[
    \Kinf(p,u,f) \triangleq \inf\left\{  \KL(p,q): q\in\simplex_{m}, qf \geq u\right\}\,,
 \]
where we recall that $pf \triangleq \E_{X\sim p} f(X)$. This quantity  appears already in the analysis of non-parametric bounded multi-arm stochastic bandits, see \citet{honda2010asymptotically,KLUCBJournal}. As the Kullback-Leibler divergence, the minimum  Kullback-Leibler divergence admits a variational formula by Lemma 18 of \citet{garivier2018kl} up to rescaling for any $u \in (0, b)$,
\begin{equation}\label{eq:kinf_variational}
    \Kinf(p,u,f) = \max_{\lambda \in[0,1/(b-u)]} \E_{X\sim p}\left[ \log\left( 1-\lambda (f(X)-u)\right)\right] \,.
 \end{equation}
\begin{theorem}[Exponential upper bound, see Theorem~\ref{thm:upper_bound_dbc}]\label{thm:expon_ub}
    For any $\alpha = (\alpha_0, \alpha_1, \ldots, \alpha_m) \in \R_{++}^{m+1}$ define  $\up \in \simplex_{m}$ such that $\up(\ell) = \alpha_\ell/\ualpha, \ell = 0, \ldots, m$, where $\ualpha = \sum_{j=0}^m \alpha_j$. Then for any $f \colon \{0,\ldots,m\} \to [0,b]$ and $0 < \mu < b,$ we have
  \[
        \P_{w\sim \Dir(\alpha)} [wf \geq \mu] \leq \exp(-\ualpha \Kinf(\up,\mu, f)).
  \]
\end{theorem}

The second result is devoted to a tight Gaussian lower bound for the distribution of a linear function of Dirichlet random vector. Here we follow the ideas of \cite{alfers1984normal} and use the exact expression for the density of a linear form of Dirichlet random vector derived by \cite{tiapkin22dirichlet}.
\begin{theorem}[Gaussian lower bound, see Theorem~\ref{thm:lower_bound_dbc}]\label{thm:gaussian_lb}
    For any $\alpha = (\alpha_0+1, \alpha_1, \ldots, \alpha_m) \in \R_{++}^{m+1},$ define  $\up \in \simplex_{m}$ such that $\up(\ell) = \alpha_\ell/\ualpha, \ell = 0, \ldots, m$, where $\ualpha = \sum_{j=0}^m \alpha_j$. Fix $\varepsilon \in (0,1)$ and assume that \(
        \alpha_0 \geq c(\varepsilon) + \log_{17/16}(\ualpha) \) for $c(\varepsilon)$ defined in \eqref{eq:c0_eps}, Appendix~\ref{app:gaussian_approximation}, and $\ualpha \geq 2\alpha_0$. Then for any $f \colon \{0,\ldots,m\} \to [0,\ub]$ such that $f(0) = \ub$, $f(j) \leq b < \ub/2, j \in \{1,\ldots,m\}$ and $\mu \in (\up f,  \ub),$ 
        \[
            \P_{w \sim \Dir(\alpha)}[wf \geq \mu] \geq (1 - \varepsilon)\P_{g \sim \cN(0,1)}\left[g \geq \sqrt{2 \ualpha \Kinf(\up, \mu, f)}\right].
        \]
\end{theorem}

We emphasize that  increasing the parameter $\alpha_0$  corresponding to the largest value of $f$ by $1$ is  crucial. The same technique was used by \cite{alfers1984normal} to derive a lower bound on the tails of the Beta distribution.

\begin{remark}\label{rem:bound_icml_isufficent}
We stress that the anti-concentration inequality of \citet[Theorem~D.2]{tiapkin22dirichlet} is not sufficient for our purposes; their additional factor $\ualpha^{\,-3/2}$  in front of the exponent  makes it unusable for the analysis of \OPSRL. Indeed, this  inequality  would  imply $\tcO(T^{3/2})$ samples from the inflated posterior in order to get optimism with high-probability,  whereas with our refined bound (Theorem~\ref{thm:gaussian_lb}) we only need  $\tcO(1)$ posterior samples. 
\end{remark}

\begin{proof}[Proof sketch of Theorem~\ref{thm:gaussian_lb}] 
    We start from the integral representation for the density by \citet[Proposition D.3]{tiapkin22dirichlet}. Define $Z \triangleq wf$ for $w \sim \Dir(\alpha_0+1, \alpha_1, \ldots, \alpha_m)$, then for any $u \in (0, \ub),$
    \[
        p_{Z}(u) = \frac{\ualpha}{2\pi} \int_\R (1 + \rmi(\ub - u)s)^{-1} \prod_{j=0}^m \left( 1 + \rmi(f(j) - u)s \right)^{-\alpha_j} \rmd s.
    \]
    One additional term $(1 + \rmi(\ub - u)s)^{-1}$ comes from  increasing  the parameter $\alpha_0$ by $1$ corresponding to the value $f(0) = \ub$.
    
    In the same spirit as it was done by \citet{tiapkin22dirichlet}, we apply the method of saddle point (see \citealp{fedoryuk1977metod,olver1997asymptotics}) to the complex integral above. Informally, for $\alpha_0, \ualpha, \ub$ large enough the following approximation holds
    \[
        p_Z(u) \approx \sqrt{\frac{\ualpha}{2\pi \sigma^2 (1 - \lambda^\star(\ub - u))^2}} \exp(-\ualpha \Kinf(\up, u, f)),
    \]
    where $\lambda^\star$ is the unique solution to the problem \eqref{eq:kinf_variational} and $\sigma^2 = \E_{X \sim \up}\big[(\frac{f(X) - u}{1 - \lambda^\star (f(X) - u)})^2\big]$. The formal statement can be found in Lemma~\ref{lem:lb_dirichlet_density} of Appendix~\ref{app:gaussian_approximation}.
    
    Next we perform a change of variables $t^2/2 = \Kinf(\up, u, f)$ in the above expression  to get
    \begin{align*}
        \P_{w \sim \Dir(\alpha_0+1,\alpha_1,\ldots,\alpha_m)}[wf \geq \mu] &\approx \int_{\mu}^{\ub} \sqrt{\frac{\ualpha}{2\pi \sigma^2 (1 - \lambda^\star(\ub - u))^2}} \exp(-\ualpha \Kinf(\up, u, f)) \rmd u \\
        &\approx \int_{\sqrt{2\Kinf(\up, \mu, f)}}^{\infty} D(u(t)) \phi( t | 0, \ualpha) \rmd t,
    \end{align*}
    where $\phi(x | \mu, \sigma^2)$ is a density of $\cN(\mu, \sigma^2)$ and $D(u)$ is a  weight function bounded from below by $1$ (see Lemma~\ref{lem:kinf_lower_bound} of Appendix~\ref{app:gaussian_approximation}). This lower bound on $D(u)$ concludes the proof.
\end{proof}

\paragraph{Comparison with anti-concentration bound by \cite{agrawal2020posterior}} We emphasise that our technique of deriving a Gaussian-like lower bound is substantially different from the methodology used  by \cite{agrawal2020posterior}. The latter one was based on reduction of a weighted sum of Dirichlet random vector to a weighted sum of independent Beta distributed random variables and a subsequent application of the Berry-Esseen inequality, whereas our approach  relies on the integral representation for the density of the corresponding  linear projection of  Dirichlet random vector.

In particular, the Berry-Esseen inequality is likely to be very coarse since it uses only the first three moments of the distribution and therefore generates an additional $S$-factor. At the same time, our analysis is much better fitted to the Dirichlet distribution and provides a very tight lower bound. The tightness of our bounds can be checked by comparing it to a similar result for the beta distribution derived in \cite{alfers1984normal}.

\paragraph{Step 2. Optimism}
Next, we apply Theorem~\ref{thm:gaussian_lb} to prove that the estimate of Q-function $\uQ^t_h$ is optimistic with high probability for our choice of inflation parameter $\kappa$ and a number of posterior samples $J$: $\uQ^t_h(s,a) \geq \Qstar_h(s,a)$  for any $(s,a,h,t) \in \cS \times \cA \times [H] \times [T]$.

We show that    the inequalities \(\max_{j \in [J]} \{ \tp^{\,t,j}_h \Vstar_{h+1}(s,a) \} \geq p_h \Vstar_{h+1}(s,a)\) hold for all $(s,a,h,t) \in \cS \times \cA \times [H] \times [T]$  with high probability. First, we notice that $\tp^{\,t,j}_h(s,a) \sim \Dir(\alpha_0 + 1, \alpha_1,\ldots, \alpha_S)$ for $\alpha_0 = n_0/\kappa - 1, \alpha_i = n^t_h(s_i|s,a)/\kappa$ and $\ualpha = (\upn^{\,t}_h(s,a) - \kappa)/\kappa$. Additionally, define a probability distribution $q \in \simplex_S$ such that $q(i) = \alpha_i / \ualpha$. This distribution slightly differs from $\up^{\,t}_h(s,a)$ because of an additional $+1$ in the parameters of the Dirichlet distribution.
Next, we may apply Theorem~\ref{thm:gaussian_lb} with $\varepsilon = 1/2$ and a proper choice of $n_0=n_0(\varepsilon),$
\[
    \P_{\tp^{\,t,j}_h(s,a) \sim \Dir(\alpha_0 + 1, \alpha_1,\ldots, \alpha_S)}\left[\tp^{\,t,j}_h \Vstar_{h+1}(s,a) \geq p_h \Vstar_{h+1}(s,a) \right] \geq \frac{1}{2}\left( 1 - \Phi\left(\sqrt{\frac{2\zeta}{\kappa}}\right) \right),
\]
where $\zeta \triangleq (\upn^t_h - \kappa) \Kinf(q, p_h \Vstar_{h+1}(s,a), \Vstar_{h+1})$ and $\Phi(\cdot)$ is a cumulative distribution function (CDF) of a standard normal distribution. By a concentration argument we have 
\[
    \zeta  \leq n^t_h \Kinf(\hp^{\,t}_h(s,a), p_h \Vstar_{h+1}(s,a), \Vstar_{h+1}) \leq \kappa/2,
\] 
with high probability for an appropriate choice of $\kappa = \tcO(1)$. For this step of the proof the presence of the inflation parameter $\kappa$ is crucial: this parameter increases the variance of $\tp^{\,t,j}_h (s,a)$  to ensure that the above inequality  holds with a constant probability.
Next, by taking the maximum over $J = \cO(\log(SATH/\delta))$ samples and applying union bound, we guarantee that the  inequality $\max_{j \in [J]} \{ \tp^{\,t,j}_h \Vstar_{h+1}(s,a) \} \geq p_h \Vstar_{h+1}(s,a)$ holds simultaneously for all $(s,a,h,t) \in \cS \times \cA \times [H] \times [T]$ with probability at least $1-\delta/2$. The formal statement and the proof could be found in Proposition~\ref{prop:anticonc} of Appendix~\ref{app:optimism}. 

Finally, the standard backward induction over $h \in [H]$ concludes optimism.  Indeed, the base of induction $h=H+1$ is trivial. Next, by the Bellman equations for $\uQ^t_h$ and $\Qstar_h$ we have
\[
    \uQ^t_h(s,a) - \Qstar_h(s,a) = \max_{j \in [J]} \{ \tp^{\,t,j}_h \uV^t_{h+1}(s,a) \} - p_h \Vstar_{h+1}(s,a).
\]
The induction hypothesis implies $\uV^t_{h+1}(s') \geq \uQ^t_{h+1}(s', \pi^\star(s')) \geq \Qstar_{h+1}(s', \pi^\star(s')) = \Vstar_{h+1}(s')$ for any $s' \in \cS$. Hence, 
\[
    \uQ^t_h(s,a) - \Qstar_h(s,a) \geq \max_{j \in [J]} \{ \tp^{\,t,j}_h \Vstar_{h+1}(s,a) \} - p_h \Vstar_{h+1}(s,a) \geq 0
\]
with probability at least $1-\delta/2.$

\paragraph{Step 3. Regret bound}
The rest of proof directly follows  \cite{azar2017minimax}, where  \UCBVI  algorithm with Bernstein bonuses was analyzed. 
By the optimism, we have
\[
    \regret^T = \sum_{t=1}^T [\Vstar_1(s_1) - V^{\pi^t}_1(s_1)]  \leq \sum_{t=1}^T \delta^t_1,
\]
where $\delta^t_h \triangleq \uV^{t-1}_h(s^t_h) - V^{\pi^t}_h(s^t_h)$. The quantity $\delta^t_h$ can be decomposed as follows using the Bellman equation for $V^{\pi^t}$ and $\uQ^{t-1}_h,$

\begin{align*}
    \delta^t_h &= \uQ^{t-1}_h(s^t_h, a^t_h) - Q^{\pi^t}_h(s^t_h, a^t_h) =  \max_{j \in [J]}\left\{ \tp^{\,t-1,j}_h \uV^{t-1}_{h+1}(s^t_h, a^t_h) \right\} - p_h V^{\pi^t}_{h+1}(s^t_h, a^t_h) \\
    &= \underbrace{\max_{j \in [J]}\left\{ \tp^{\,t-1,j}_h \uV^{t-1}_{h+1}(s^t_h, a^t_h) \right\} - \up^{\,t-1}_h \uV^{t-1}_{h+1}(s^t_h, a^t_h)}_{\termA} + \underbrace{[\up^{\,t-1}_h - \hp^{\,t-1}_h] \uV^{t-1}_{h+1}(s^t_h, a^t_h)}_{\termB} \\
    &+ \underbrace{[\hp^{\,t-1}_h - p_h] [\uV^{t-1}_{h+1} - \Vstar_{h+1}] (s^t_h, a^t_h)}_{\termC} +  \underbrace{[\hp^{\,t-1}_h - p_h] \Vstar_{h+1}(s^t_h, a^t_h)}_{\termD} \\
    & + \underbrace{p_h[\uV^{t-1}_{h+1} - V^{\pi^t}_{h+1}](s^t_h, a^t_h) - [\uV^{t-1}_{h+1} - V^{\pi^t}_{h+1}](s^t_{h+1})}_{\xi^t_h} + \delta^t_{h+1}.
\end{align*}
The terms $\termC, \termD,$ and $\xi^t_h$ are standard in the analysis of the optimistic algorithms. The term $\termB$ could be upper-bounded by $\frac{r_0 \cdot n_0 \cdot H}{\upn^{t-1}_h(s^t_h, a^t_h)}$ and turns out to be one of second-order terms. The analysis of $\termA$ is novel and requires application of the Bernstein inequality for Dirichlet distributions that follows from Theorem~\ref{thm:expon_ub} and is spelled out in the following lemma.
\begin{lemma}[see Lemma~\ref{lem:bernstein_dirichlet} in Appendix~\ref{app:deviation_ineq}]
     For any $\alpha = (\alpha_0, \alpha_1, \ldots, \alpha_m) \in \R_{++}^{m+1}$ define  $\up \in \simplex_{m}$ such that $\up(\ell) = \alpha_\ell/\ualpha, \ell = 0, \ldots, m$, where $\ualpha = \sum_{j=0}^m \alpha_j$. Then for any $f \colon \{0,\ldots,m\} \to [0,b]$ such that $f(0) = b$ and $\delta \in (0,1),$
     \[
        \P_{w \sim \Dir(\alpha)}\left[wf \geq \up f +  2 \sqrt{ \frac{ \Var_{\up}(f) \log(1/\delta)}{\ualpha}} + \frac{3b \cdot \log(1/\delta)}{\ualpha} \right] \leq \delta.
     \]
\end{lemma}
As opposed to  Lemma~C.8 of \citet{tiapkin22dirichlet}, the last result applies  to Dirichlet distributions with non-integer parameters as in our case (due to the presence of the inflation parameter $\kappa$). Therefore, we see that the term $\termA$ can be upper bounded by a quantity which has the same role as  in the analysis of \UCBVI. After using the Bernstein bound, the rest of the proof follows from the analysis of \UCBVI with the Bernstein bonuses and \BayesUCBVI; see \citet{azar2017minimax} and \citet{tiapkin22dirichlet}.

\section{Experiments}
\label{sec:experiments}

In this section we provide experiment to compare \OPSRL with some baselines on simple tabular environment; see details in Appendix~\ref{app:experiments}. In particular, we illustrate that \OPSRL is competitive with the original \PSRL algorithm and outperforms bonus-based algorithms such as \UCBVI. 

\paragraph{Baselines}
We compare \OPSRL with the following baselines: \UCBVI (with Hoeffding-type bonuses) and \UCBVIB (with Bernstein-type bonuses) \cite{azar2017minimax}, \PSRL \cite{obsband2013more}, and \RLSVI \cite{osband16generalization}. See Appendix~\ref{app:experiments} for full details on parameters for \OPSRL and baselines.

\begin{figure}[h!]
      \vspace{-0.2cm}
    \centering
    \includegraphics[keepaspectratio,width=.6\textwidth]{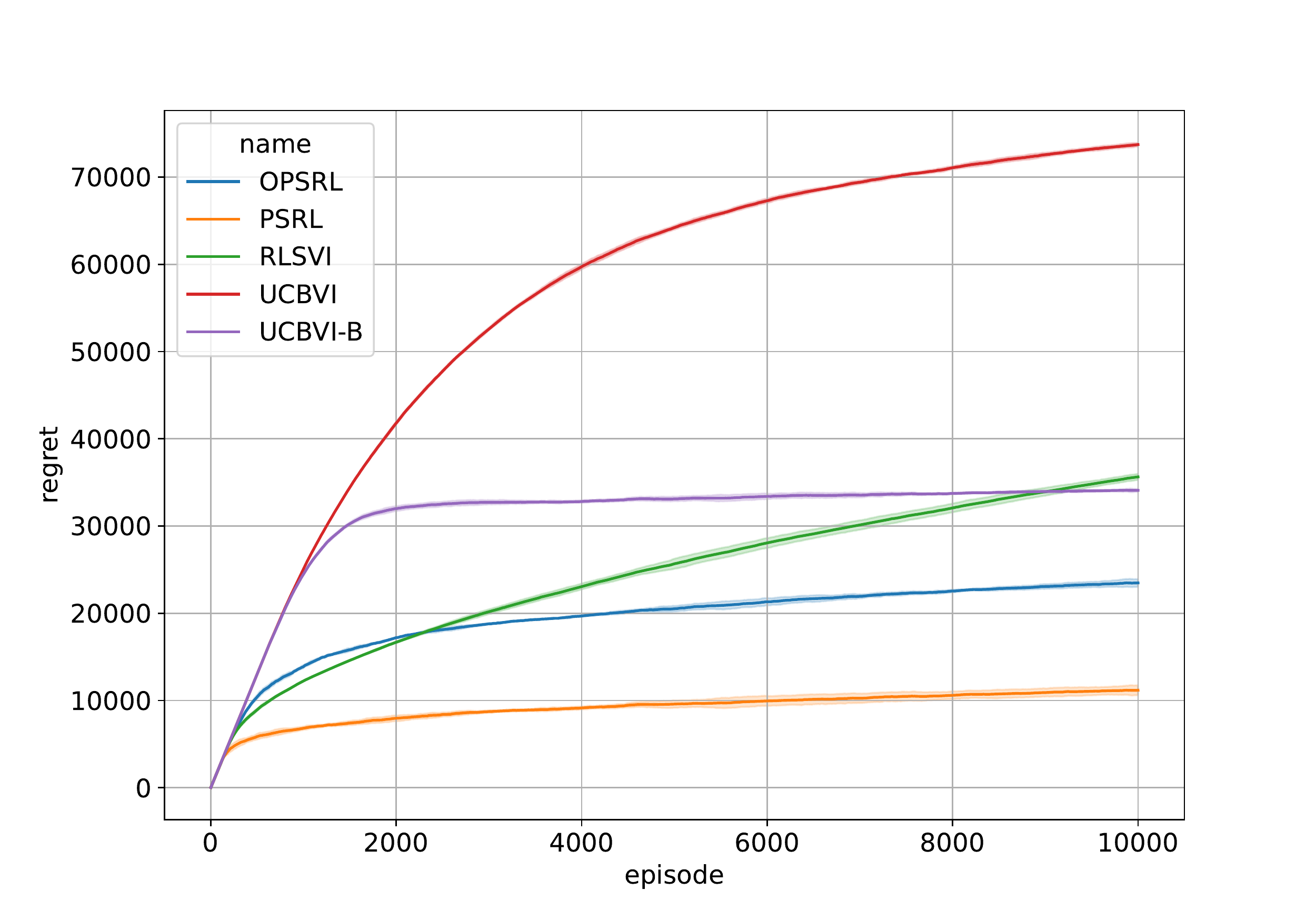}
    \caption{Regret of \OPSRL and baselines on grid-world environment with $100$ states and $4$ action for $H=50$ an transitions noise $0.2$. We show $\text{average}$ over $4$ seeds.}
    \label{fig:regret_baselines}
\end{figure}

\paragraph{Results} In Figure~\ref{fig:regret_baselines}, we plot the regret of the various baselines and \OPSRL in the grid world environment. In this experiment, we observe that \OPSRL achieves competitive results with respect to \PSRL. It is not completely surprising since they share the same Bayesian model on the transitions up to the prior. We shall elaborate more on the influence of the prior in Appendix~\ref{app:experiments}. We also note that \OPSRL outperforms \UCBVI and \RLSVI.
This difference may be explained by the fact that \OPSRL's optimism implies (in the worst case) KL bonuses as in \citet{filippi2010optimism}. The KL bonuses are stronger than Bernstein bonuses, see Lemma~\ref{lem:Bernstein_via_kl}, because they somehow rely on all moments of the empirical distribution rather than the first two moments as in the case of Bernstein bonuses or first moments for Hoeffding bonuses or for the variance of the Gaussian noise in \RLSVI. Note also that in \OPSRL, we do not have to solve the complex convex program to compute the KL bonuses \citet{filippi2010optimism}, which could be computationally intensive. 

\section{Conclusion}
\label{sec:conclusion}

In this work, we presented the \OPSRL algorithm which can be viewed as a simple optimistic variant of the \PSRL algorithm. Notably, \OPSRL only needs $\tcO(1)$ posterior samples per state-action. We proved that the regret of \OPSRL is upper-bounded with high probability by $\tcO(\sqrt{H^3SAT}),$ matching the problem-independent lower-bound of order $\Omega(\sqrt{H^3SAT})$ for $T$ large enough and up to terms poly-logarithmic  in $H,S,A,$ and $T$. While our work addresses the open questions raised by \citet{agrawal2020posterior} in the episodic setting, obtaining the same results in the infinite-horizon average reward setting remains an open issue. We believe that it is possible to adapt our analysis to this other setting up to some technical adjustments. Ultimately, another open question, is to obtain a high-probability regret bound for \PSRL, that is, when using only a \textit{single} posterior sample and not inflating the posterior.
As a further future research direction we believe it could be interesting to obtain a model-free algorithm that relies on the same mechanism as \OPSRL for exploration. Indeed, such an algorithm could avoid the use of complicated bonuses adopted by the current model-free algorithms while reducing the memory complexity of \OPSRL.
\vspace{2cm}

\vspace{-2cm}
\begin{ack}
    D.\,Belomestny acknowledges the financial support from Deutsche Forschungsgemeinschaft (DFG), Grant Nr.\,497300407. The work of D.\,Tiapkin, D.\,Belomestny and A.\,Naumov was prepared within the framework of the HSE University Basic Research Program. Pierre Ménard acknowledges the support of the Chaire SeqALO (ANR-20-CHIA-0020-01).
\end{ack}

\bibliographystyle{plainnat}
\bibliography{biblio.bib}

\newpage
\appendix

\part{Appendix}
\parttoc
\newpage
\section{Notation}
\label{app:notations}

\begin{table}[h]
	\centering
	\caption{Table of notation use throughout the paper}
	\begin{tabular}{@{}l|l@{}}
		\toprule
		\thead{Notation} & \thead{Meaning} \\ \midrule
	$\cS$ & state space of size $S$\\
	$\cA$ & action space of size $A$\\
	$H$ & length of one episode\\
	$T$ & number of episodes\\
	$J$ & number of posterior samples\\
	\hline
	$r_h(s,a)$ & reward \\
	$p_h(s'|s,a)$ & probability transition \\
	$Q^{\pi}_h(s,a)$ & Q-function of a given policy $\pi$ at step $h$\\
	$V^{\pi}_h(s)$ & V-function of a given policy $\pi$ at step $h$\\
	$\Qstar_h(s,a)$ & optimal Q-function at step $h$\\
	$\Vstar_h(s)$ & optimal V-function at step $h$ \\
	$\regret^T $ & regret \\
	\hline
	$n_0$ & number of pseudo-samples \\
	$s_0$ & pseudo-state \\
	$\ur$ & pseudo-reward \\
	$\kappa$ & posterior inflation parameter \\
	\hline
	$s^{\,t}_h$ & state that was visited at $h$ step during $t$ episode \\
	$a^{\,t}_h$ & action that was picked at $h$ step during $t$ episode \\
	$n_h^t(s,a)$ & number of visits of state-action $n_h^t(s,a) = \sum_{k = 1}^t  \ind{\left\{(s_h^k,a_h^k) = (s,a)\right\}}$\\
	$n_h^t(s'|s,a)$ & number of transition to $s'$ from state-action $n_h^t(s'|s,a) = \sum_{k = 1}^t  \ind{\left\{(s_h^k,a_h^k, s_{h+1}^k) = (s,a,s')\right\}}$. \\
	$\upn_h^t(s,a)$ & pseudo number of visits of state-action $\upn_h^t(s,a)=n_h^t(s,a)+n_0$\\
	$\upn_h^t(s'|s,a)$ & pseudo number of  transition to $s'$ from state-action $\upn_h^t(s'|s,a)=n_h^t(s'|s,a) + \ind{\{s' = s_0\}} \cdot n_0$\\
	$\hp_h^{\,t}(s'|s,a)$ & empirical probability transition $\hp_h^{\,t}(s'|s,a) = n_h^t(s'|s,a) / n_h^t(s,a)$ \\
	$\up_h^t(s'|s,a)$ & pseudo-empirical probability transition $\up_h^t(s'|s,a) = \upn_h^t(s'|s,a) / \upn_h^t(s,a)$ \\
	\hline
    $\uQ_h^t(s,a)$ & upper approximation of the optimal Q-value\\
    $\uV_h^t(s,a)$ & upper approximation of on the optimal V-value\\
    \hline 
    $\R_+$ & non-negative real numbers \\
    $\R_{++}$ & positive real numbers \\
    $\N_{++}$ & positive natural numbers \\
    $[n]$ & set $\{1,2,\ldots, n\}$ \\
    $\simplex_d$ & $d$-dimensional probability simplex: $\simplex_d = \{x \in \R_{+}^{d+1}: \sum_{j=0}^{d} x_j = 1 \}$ \\ 
    $\bOne^N$ & vector of dimension $N$ with all entries one is $\bOne^N \triangleq  (1,\ldots,1)$ \\
    $\norm{x}_1$ & $\ell_1$-norm of vector $\norm{x}_1 = \sum_{j=1}^m \vert x_j \vert$\\
    $\norm{x}_2$ & $\ell_2$-norm of vector $\norm{x}_2 = \sqrt{\sum_{j=1}^m x_j^2}$\\
    $\norm{f}_2 $ & for $f \colon \Xset \to \R$, where $\vert \Xset \vert <\infty$ define $\norm{f}_2 = \sqrt{\sum_{x \in \Xset} f^2(x)}$ \\
    \bottomrule
	\end{tabular}
\end{table}

Let $(\Xset,\Xsigma)$ be a measurable space and $\Pens(\Xset)$ be the set of all probability measures on this space. For $p \in \Pens(\Xset)$ we denote by $\E_p$ the expectation w.r.t.\,$p$. For random variable $\xi: \Xset \to \R$ notation $\xi \sim p$ means $\operatorname{Law}(\xi) = p$. We also write $\E_{\xi \sim p}$ instead of $\E_{p}$. For independent (resp.\,i.i.d.) random variables $\xi_\ell \mysim p_\ell$ (resp.\,$\xi_\ell \mysimiid p$), $\ell = 1, \ldots, d$, we will write $\E_{\xi_\ell \mysim p_\ell}$ (resp.\,$\E_{\xi_\ell \mysimiid p}$), to denote expectation w.r.t.\,product measure on $(\Xset^d, \Xsigma^{\otimes d})$. For any $p, q \in \Pens(\Xset)$ the Kullback-Leibler divergence $\KL(p, q)$ is given by
$$
\KL(p, q) \triangleq \begin{cases}
\E_{p}\left[\log \frac{\rmd p}{\rmd q}\right], & p \ll q, \\
+ \infty, & \text{otherwise.}
\end{cases} 
$$
For any $p \in \Pens(\Xset)$ and $f: \Xset \to \R$, $p f = \E_p[f]$. In particular, for any $p \in \simplex_d$ and $f: \{0, \ldots, d\}   \to  \R$, $pf =  \sum_{\ell = 0}^d f(\ell) p(\ell)$. Define $\Var_{p}(f) = \E_{s' \sim p} \big[(f(s')-p f)^2\big] = p[f^2] - (pf)^2$. For any $(s,a) \in \cS$, transition kernel $p(s,a) \in \Pens(\cS)$ and $f \colon \cS \to \R$ define $pf(s,a) = \E_{p(s,a)}[f]$ and $\Var_{p}[f](s,a) = \Var_{p(s,a)}[f]$.

We write $f(S,A,H,T) = \cO(g(S,A,H,T,\delta))$ if there exist $ S_0, A_0, H_0, T_0, \delta_0$ and constant $C_{f,g}$ such that for any $S \geq S_0, A \geq A_0, H \geq H_0, T \geq T_0, \delta < \delta_0, f(S,A,H,T,\delta) \leq C_{f,g} \cdot g(S,A,H,T,\delta)$. We write $f(S,A,H,T,\delta) = \tcO(g(S,A,H,T,\delta))$ if $C_{f,g}$ in the previous definition is poly-logarithmic in $S,A,H,T,1/\delta$.

For $\lambda > 0,$ we define $\Exponential(\lambda)$ as an exponential distribution with a parameter $\lambda$. For $k, \theta > 0$ define $\Gamma(k,\theta)$ as a gamma-distribution with a shape parameter $k$ and a rate parameter $\theta$. For set $\Xset$ such that $\vert \Xset \vert < \infty$ define $\Unif(\Xset)$ as a uniform distribution over this set. In particular, $\Unif[N]$ is a uniform distribution over a set $[N]$.

We fix a function $f: \{1,\ldots,m\} \mapsto [0,b]$ and recall the definition of the minimum Kullback-Leibler divergence for $p\in\simplex_{m-1}$  and $u\in\R$
 \[
\Kinf(p,u,f) \triangleq \inf\left\{  \KL(p,q): q\in\simplex_{m-1}, qf \geq u\right\}\,.
 \]
As the Kullback-Leibler divergence this quantity admits a variational formula by Lemma~18 of \citet{garivier2018kl} up to rescaling for any $u \in (0, b)$
\[
\Kinf(p,u,f) = \max_{\lambda \in[0,1/(b-u)]} \E_{X\sim p}\left[ \log\left( 1-\lambda (f(X)-u)\right)\right] \,.
 \]
\newpage
\section{Proof of regret bound for \OPSRL}
\label{app:regret_bound_proof}
\subsection{Concentration events}
\label{app:concentration}

Let $\betastar,  \beta^{\KL}, \beta^{\conc}, \beta^{\Var} \colon (0,1) \times \N \to \R_{+}$, $\beta^{\Dir} \colon (0,1) \times \N \times \N \to \R_{+}$, and $\beta \colon (0,1) \to \R_{+}$ be some function defined later  in Lemma~\ref{lem:proba_master_event}. We define the following favorable events,
\begin{align*}
  \cE^\star(\delta) &\triangleq \Bigg\{\forall t \in \N, \forall h \in [H], \forall (s,a)\in\cS\times\cA: \\
  &\qquad
    \Kinf(\hp_h^t(s,a),p_h \Vstar_{h+1}(s,a), \Vstar_{h+1}) \leq  \frac{\betastar(\delta,n_h^t(s,a))}{n_h^t(s,a)}\Bigg\}\CommaBin\\
\cE^{\KL}(\delta) &\triangleq \Bigg\{ \forall t \in \N, \forall h \in [H], \forall (s,a) \in \cS\times\cA: \\
&\qquad \KL(\hp^{\,t}_h(s,a), p_h(s,a)) \leq \frac{S \cdot \beta^{\KL}(\delta, n^{\,t}_h(s,a))}{n^{\,t}_h(s,a)} \Bigg\}\CommaBin\\
\cE^{\conc}(\delta) &\triangleq \Bigg\{\forall t \in \N, \forall h \in [H], \forall (s,a)\in\cS\times\cA: \\
&\qquad|(\hp_h^t -p_h) \Vstar_{h+1}(s,a)| \leq \sqrt{2 \Var_{p_h}(\Vstar_{h+1})(s,a)\frac{\beta(\delta,n_h^t(s,a))}{n_h^t(s,a)}} + 3 H \frac{\beta(\delta,n_h^t(s,a))}{n_h^t(s,a)}\Bigg\}\CommaBin\\
\cE^{\Dir}(\delta) &\triangleq \Bigg\{ \forall t \in [T], \forall h \in [H], \forall (s,a) \in \cS \times \cA, \forall j \in [J]: \\
&\qquad [\tp^{t,j}_h - \up^t_h] \uV^t_{h+1}(s,a) \leq 2 \sqrt{  \Var_{\up^{\,t}_h}[\uV^t_{h+1}](s,a) \frac{\beta^{\Dir}(\delta, T, J) \cdot \kappa }{\upn^{\,t}_h(s,a)}} + 3r_0H\frac{\beta^{\Dir}(\delta, T, J) \cdot \kappa }{\upn^{\,t}_h(s,a)} \Bigg\}\CommaBin\\
\cE^{\Var}(\delta) &\triangleq \Bigg\{\forall t \in \N: \quad \sum_{\ell=1}^t \sum_{h=1}^H \Var_{p_h}[V^{\pi_\ell}_{h+1}(s^\ell_h, a^\ell_h)] \leq H^2t + \sqrt{2H^5 t \beta^{\Var}(\delta, t)} + 3H^3 \beta^{\Var}(\delta, t)
\Bigg\}\CommaBin\\
\cE(\delta) &\triangleq \Bigg\{ \sum_{t=1}^T \sum_{h=1}^H \left| p_h[\uV^{t-1}_{h+1} - V^{\pi^t}_{h+1}](s^t_h, a^t_h) - [\uV^{t-1}_{h+1} - V^{\pi^t}_{h+1}](s^t_{h+1}) \right| \leq 2\ur H\sqrt{2HT \beta(\delta)},\\
    &\qquad \sum_{t=1}^T \sum_{h=1}^H (1-1/H)^{H-h+1}\biggl| p_h[\uV^{t-1}_{h+1} - V^{\pi^t}_{h+1}](s^t_h, a^t_h)  \\
    &\qquad\qquad\qquad\qquad\qquad\qquad\quad - [\uV^{t-1}_{h+1} - V^{\pi^t}_{h+1}](s^t_{h+1})\biggl|\leq 2\rme \ur H\sqrt{2HT \beta(\delta)},
\Bigg\}\cdot
\end{align*}
We also introduce the intersection of these events, $\cG^{\conc}(\delta) \triangleq \cE^\star(\delta) \cap \cE^{\KL}(\delta) \cap \cE^{\conc}(\delta) \cap \cE^{\Dir}(\delta) \cap \cE^{\Var}(\delta) \cap \cE(\delta)$. We  prove that for the right choice of the functions $\betastar,  \beta^{\KL}, \beta^{\conc}, \beta, \beta^{\Var},$ the above events hold with high probability.
\begin{lemma}
\label{lem:proba_master_event}
For any $\delta \in (0,1)$ and for the following choices of functions $\beta,$
\begin{align*}
    \betastar(\delta,n) &\triangleq \log(12SAH/\delta) + 3\log\left(\rme\pi(2n+1)\right)\,,\\
    \beta^{\KL}(\delta, n) & \triangleq \log(12SAH/\delta) + \log\left(\rme(1+n) \right), \\
    \beta^{\conc}(\delta, n) &\triangleq \log(12SAH/\delta) + \log(4\rme(2n+1)) ,\\
    \beta^{\Dir}(\delta,t, J) &\triangleq \log(12SAHt/\delta) + \log(J), \\
    \beta^{\Var}(\delta, t) &\triangleq \log(48\rme(2t+1)/\delta),\\
    \beta(\delta) &\triangleq \log\left(48/\delta\right),\\
\end{align*}
it holds that
\begin{align*}
\P[\cE^\star(\delta)]&\geq 1-\delta/12, & 
\P[\cE^{\KL}(\delta)]&\geq 1-\delta/12,  &
\P[\cE^\conc(\delta)]&\geq 1-\delta/12,\\
\P[\cE^{\Dir}(\delta)]&\geq 1-\delta/12, &
\P[\cE^\Var(\delta)]&\geq 1-\delta/12, &
\P[\cE(\delta)]&\geq 1-\delta/12.
\end{align*}
In particular, $\P[\cG^{\conc}(\delta)] \geq 1-\delta/2$.
\end{lemma}
\begin{remark}
    Since we take $J \triangleq \Theta(\log(SAHT/\delta))$, all functions $\beta$ are logarithmic in $S,A,H,T,\delta$.
\end{remark}

\begin{proof}
     $\P[\cE^\star(\delta)]\geq 1-\delta/12$ follows from Theorem~\ref{th:max_ineq_kinf}. Applying Theorem~\ref{th:max_ineq_categorical} and the union bound over $h \in [H], (s,a) \in \cS \times \cA$ we get $\P[\cE^{\KL}(\delta)]\geq 1-\delta/12$. Next, by Lemma~\ref{lem:bernstein_dirichlet} and the union bound over $h \in [H], t \in [T], (s,a) \in \cS \times \cA, j \in [J]$ we conclude $\P[\cE^{\Dir}(\delta)]\geq 1-\delta/12$. Theorem~\ref{th:bernstein} and the union bound over $h \in [H], (s,a) \in \cS \times \cA$ yield $\P[\cE^{\conc}(\delta)]\geq 1 - \delta/12$. By Lemma~B.2 by \citet{tiapkin22dirichlet} we have $\P[\cE^{\Var}(\delta)]\geq 1 - \delta/12$.
     
     To estimate $\P[\cE(\delta)]$ one may apply Azuma-Hoeffding inequality. Define the following sequences for all $t \in [T], h \in [H]$
     \begin{align*}
        \bar Z_{t,h} &\triangleq \uV_{h+1}^{t-1}(s_{h+1}^t)-V^*_{h+1}(s_{h+1}^t)-p_h [\uV_{h+1}^{t-1}- V^*_{t+1}](s^t_h,a^t_h),\\
        \tilde Z_{t,h} &\triangleq (1-1/H)^{H-h+1}\left(\uV_{h+1}^{t-1}(s_{h+1}^t)-V^*_{h+1}(s_{h+1}^t)-p_h [\uV_{h+1}^{t-1}- V^*_{h+1}](s^t_h,a^t_h)\right).
\end{align*}
    It is easy to see that these sequences form a martingale-difference w.r.t filtration $\cF_{t,h} = \sigma\left\{ \{ (s^{\ell}_{h'}, a^{\ell}_{h'}), \ell < t, h' \in [H] \} \cup \{ (s^{t}_{h'}, a^t_{h'}), h' \leq h \} \right\}$. Moreover,  \(|\bar Z_{t,h}|\leq 2\ur H, |\tilde Z_{t,h}\vert \leq 2\rme \ur H\) for all \(t\in [T]\) and \(h\in [H].\)   Hence, the Azuma-Hoeffding inequality implies
    \begin{align*}
        \P\left(\Bigl|\sum_{t=1}^T \sum_{h=1}^H \bar Z_{t,h}\Bigr|> 2\ur H\sqrt{2 t H \cdot \beta(\delta)}\right)&\leq 2\exp(-\beta(\delta))=\delta/24, \\
        \P\left(\Bigl|\sum_{t=1}^T \sum_{h=1}^H \bar Z_{t,h}\Bigr|> 2\rme \ur H\sqrt{2 t H \cdot \beta(\delta)}\right)&\leq 2\exp(-\beta(\delta))=\delta/24.
    \end{align*}
    By the union bound, $\P[\cE(\delta)] \geq 1 - \delta/12$.
\end{proof}

Next we reproduce proof of important corollary of Lemma~\ref{lem:proba_master_event}.
\begin{lemma}\label{lem:f_and_l1_concentration}
      Assume conditions of Lemma \ref{lem:proba_master_event}. Then on the event $\cE^{\KL}(\delta)$, for any $f \colon \cS \to [0, \ur H]$, $t \in \N, h \in [H], (s,a) \in \cS \times \cA$,
      \begin{align*}
            (\hp_h^t -p_h)f(s,a) &\leq \frac{1}{H} p_h f(s,a) + \frac{5 \ur H^2 S \cdot \beta^{\KL}(\delta, n^{\,t}_h(s,a))}{n^{\,t}_h(s,a)}\CommaBin \\
            \norm{\hp^{\,t}_h(s,a) - p_h(s,a)}_1 &\leq \sqrt{\frac{2S \cdot \beta^{\KL}(\delta, n^{\,t}_h(s,a))}{n^{\,t}_h(s,a)}}\cdot
      \end{align*}
\end{lemma}
\begin{proof}
    By application of Lemma~\ref{lem:Bernstein_via_kl} and Lemma~\ref{lem:switch_variance_bis}
    \begin{align*}
        (\hp_h^t -p_h)f(s,a) &\leq \sqrt{2\Var_{\hp_h^t}[f](s,a) \cdot \KL(\hp_h^t, p_h) } + \frac{2H\ur}{3} \KL(\hp_h^t, p_h) \\
        &\leq 2\sqrt{\Var_{p_h}[f](s,a) \cdot \KL(\hp_h^t, p_h) } + \left( 2\sqrt{2} + \frac{2}{3} \right)H\ur\KL(\hp_h^t, p_h).
    \end{align*}
    Since $0 \leq f(s) \leq \ur H$
    \[
        \Var_{p_h}[f](s,a) \leq p_h[f^2](s,a) \leq \ur H \cdot p_h f(s,a).
    \]
    Finally, by a simple bound $2\sqrt{ab} \leq a+b, a, b \geq 0$, we obtain the following
    \begin{align*}
        (\hp_h^t -p_h)f(s,a) &\leq \frac{1}{H} p_h f(s,a) + (H^2 + 2\sqrt{2} \ur H + 2\ur H/3) \KL(\hp_h^t, p_h) \\
        &\leq \frac{1}{H} p_h f(s,a) + 5 \ur H^2 \KL(\hp_h^t, p_h).
    \end{align*}
    Definition of $\cE^{\KL}(\delta)$ implies the first statement. The second statement follows directly from the combination of Pinsker's inequality and definition of $\cE^{\KL}(\delta)$.
\end{proof}

\subsection{Optimism}
\label{app:optimism}

In this section we prove that our estimate of $Q$-function $\uQ^{\,t}_h(s,a)$ is optimistic that is the event
\begin{equation}\label{eq:opt_event}
    \cE_{\opt} \triangleq \left\{ \forall t \in [T], h \in [H], (s,a) \in \cS \times \cA:  \uQ^t_h(s,a) \geq \Qstar_h(s,a) \right\}.
\end{equation}
holds with high probability on the event $\cE^\star(\delta)$.

Define constants
\begin{equation}\label{eq:constant_c0}
    c_0 \triangleq \left(\frac{4}{\sqrt{\log(17/16)}} + 8 + \frac{49 \cdot 4 \sqrt{6}}{9} \right)^2 \frac{8}{\pi} + \log_{17/16}\left( \frac{20}{32}\right) + 1,
\end{equation}
and
\begin{equation}\label{eq:constant_cJ}
    c_J \triangleq \frac{1}{\log\left( \frac{2}{1 + \Phi(1)} \right)},
\end{equation}
where $\Phi(\cdot)$ is a cdf of a normal distribution.

\begin{proposition}\label{prop:anticonc}
    Assume that $J = \lceil  c_J \cdot \log(2SAHT/\delta)  \rceil$, $\kappa = 2\beta^\star(\delta, T)$, $\ur = 2$, and $n_0 = \lceil (c_0 + \log_{17/16}(T/\kappa)) \cdot \kappa \rceil$. Then on event $\cE^\star$ the following event
    \[
        \cE^{\anticonc}(\delta) \triangleq \left\{ \forall t \in [T] \ \forall h \in [H] \ \forall (s,a) \in \cS \times \cA:  \max_{j \in [J]} \left\{ \tp^{\,t,j}_h \Vstar_{h+1}(s,a)\right\} \geq p_h \Vstar_{h+1}(s,a) \right\}
    \]
    holds with probability at least $1-\delta/2$.
\end{proposition}
\begin{proof}
    First, we notice that $\tp^{\,t,j}_h(s,a)$ for all fixed $t,j,h,s,a$ have a Dirichlet distribution with parameter $( \{\upn^t_h(s' | s,a)/\kappa\}_{s' \in \cS'} )$ for an extended state-space $\cS' = \{s_0\} \cup \cS$. Therefore, we may apply  Theorem~\ref{thm:lower_bound_dbc} with fixed $\varepsilon = 1/2$ for $f = \Vstar_{h+1}$ if we have $\ub = \ur(H-h) \geq 2(H-h) = 2b$ and
    \[
        \frac{n_0}{\kappa} = \alpha_0 + 1 \geq c_0 + \log_{17/16}\left( \upn^t_h(s,a) / \kappa \right) 
    \]
    for a constant $c_0$ defined in \eqref{eq:constant_c0}. Let us define $\alpha_0 = \nicefrac{n_0}{\kappa} - 1$ and $\alpha_{i} = \nicefrac{n^t_h(s_i|s,a)}{\kappa}$ for some ordering $s_i \in \cS$. Then we have $\ualpha = \nicefrac{\upn^t_h(s,a)}{\kappa} - 1$ and $\tp^{\,t,j}_h \sim \Dir(\alpha_0 + 1, \alpha_1,\ldots, \alpha_{S})$. Define a distribution $q \in \simplex_{S}: q(i) = \alpha_i/\ualpha$. Then if $\ualpha \geq 2\alpha_0$ Theorem~\ref{thm:lower_bound_dbc} yields for any $u \geq q \Vstar_{h+1}$
    \begin{equation}\label{eq:anticonc_1}
        \PP{ \tp^{\,t,j}_h \Vstar_{h+1}(s,a) \geq u } \geq \frac{1}{2}\left(1 - \Phi\left(\sqrt{\frac{2 (\upn^t_h(s,a) - \kappa) \Kinf\left(q, u, \Vstar_{h+1}\right) }{\kappa}}\right) \right),
    \end{equation}
    where $\Phi$ is a cdf of a normal distribution.
    
    Notice that if we have $u < q \Vstar_{h+1}$ then the following bound also holds
    \begin{equation}\label{eq:anticonc_2}
        \PP{ \tp^{\,t,j}_h \Vstar_{h+1}(s,a) \geq u } \geq 
        \PP{ \tp^{\,t,j}_h \Vstar_{h+1}(s,a) \geq \up^t_h \Vstar_{h+1}(s,a)} \geq \frac{1}{2}\left(1 - \Phi\left(0\right) \right).
    \end{equation}
    Since for all $u \leq q \Vstar_{h+1}$ we also have $ \Kinf\left(q, u, \Vstar_{h+1}\right) = 0$, therefore \eqref{eq:anticonc_1} holds for all $u \geq 0$ and $\ualpha \geq 2\alpha_0$.
    
    Next we need to handle the case $\ualpha < 2\alpha_0$. In this case we have $q \Vstar_{h+1} > H-h$, thus for any $0 \leq u \leq H-h$
    \[
        \PP{\tp^{\,t,j}_h \Vstar_{h+1}(s,a) \geq u} \geq \P_{\xi \sim B(\alpha_0+1, \ualpha - \alpha_0)}\left( \ur(H-h) \xi \geq u \right) \geq \P_{\xi \sim B(\alpha_0+1, \ualpha - \alpha_0)}\left(\xi \geq \frac{1}{2} \right),
    \]
    where we first apply a lower bound $\Vstar_{h+1}(s) \geq 0$  for all $s \in \cS$ and $\Vstar_{h+1}(s_0) = \ur(H-h)$, and second apply a bound $u \leq H-h$. Here we may apply the result of~\citet[Theorem 1.2'']{alfers1984normal} and obtain the following lower bound that is equivalent to \eqref{eq:anticonc_2}
    \[
        \PP{\tp^{\,t,j}_h \Vstar_{h+1}(s,a) \geq u} \geq \Phi\left(-\mathrm{sign}(\alpha_0/\ualpha - 1/2) \cdot \sqrt{2 \ualpha \kl(\alpha_0/\ualpha, 1/2)}\right) \geq \frac{1}{2}\left(1 - \Phi(0) \right),
    \]
    where we used $\alpha_0 / \ualpha > 1/2$.
    
    Thus, we may apply equation \eqref{eq:anticonc_1} for $u = p_h \Vstar_{h+1}(s,a) \leq H-h$ and any $\ualpha$
    \[
        \PP{ \tp^{\,t,j}_h \Vstar_{h+1}(s,a) \geq p_h \Vstar_{h+1}(s,a) } \geq \frac{1}{2}\left(1 - \Phi\left(\sqrt{\frac{2 (\upn^t_h(s,a) - \kappa) \Kinf\left(q, p_h \Vstar_{h+1}(s,a), \Vstar_{h+1}\right) }{\kappa}}\right) \right).
    \]
    By the following relation that follows from variational formula for $\Kinf$ with rescaling of $\lambda$ to $[0,1]$
    \begin{align*}
        (\upn^{\,t}_h(s,a) & - \kappa) \Kinf(q, u, \Vstar_{h+1}) = (\upn^{\,t}_h(s,a) - \kappa) \max_{\lambda \in [0,1]} \E_{s' \sim q}\left[ \log\left(1 - \lambda \frac{\Vstar_{h+1}(s') - u}{\ur(H - h) - u} \right)\right] \\
        &\leq \max_{\lambda \in [0,1]} (n_0 - \kappa) \log\left(1 - \lambda \right) + (\upn^{\,t}_h(s,a) - n_0) \max_{\lambda \in [0,1] }\E_{s' \sim \hp^{\,t}_h(s,a)}\left[ \log\left(1 - \lambda \frac{\Vstar_{h+1}(s') - u}{\ur(H - h) - u} \right) \right] \\
        &\leq (\upn^{\,t}_h(s,a) - n_0) \max_{\lambda \in [0,1] }\E_{s' \sim \hp^{\,t}_h(s,a)}\left[ \log\left(1 - \lambda \frac{\Vstar_{h+1}(s') - u}{H - h - u} \right) \right] \\
        &= (\upn^{\,t}_h(s,a) - n_0) \Kinf(\hp^{\,t}_h(s,a), u, \Vstar_{h+1}) = n^{\,t}_h(s,a) \Kinf(\hp^{\,t}_h(s,a), u, \Vstar_{h+1}).
    \end{align*}
    Thus, on the event $\cE^\star$
    \[
        (\upn^t_h(s,a) - \kappa) \Kinf\left(q, p_h \Vstar_{h+1}(s,a), \Vstar_{h+1}\right) \leq \beta^\star(\delta, n^t_h(s,a)) \leq \beta^\star(\delta, T),
    \]
    and, as a corollary
    \[
        \PP{ \tp^{\,t,j}_h \Vstar_{h+1}(s,a) \geq p_h \Vstar_{h+1}(s,a) \mid \cE^\star(\delta) } \geq \frac{1}{2} \left( 1 - \Phi\left( \sqrt{\frac{2\beta^\star(\delta, T)}{\kappa}} \right) \right).
    \]
    By taking $\kappa = 2\beta^\star(\delta, T)$ we have a constant probability of being optimistic
    \[
        \PP{ \tp^{\,t,j}_h \Vstar_{h+1}(s,a) \geq p_h \Vstar_{h+1}(s,a) \mid \cE^\star(\delta) } \geq \frac{1 - \Phi(1)}{2} \triangleq \gamma.
    \]
    Next, using a choice $J = \lceil \log(2SAHT/\delta) / \log(1/(1-\gamma)) \rceil = \lceil c_J \cdot \log(2SAHT/\delta)  \rceil$ 
    \[
        \PP{ \max_{j \in [J]}\left\{\tp^{\,t,j}_h \Vstar_{h+1}(s,a) \right\} \geq p_h \Vstar_{h+1}(s,a) \mid \cE^\star(\delta) } \geq 1 - (1 - \gamma)^{J} \geq 1 - \frac{\delta}{2SAHT}\cdot
    \]
    By a union bound we conclude the statement.
\end{proof}

Next we provide a connection between $\cE^{\anticonc}(\delta)$ and $\cE^{\opt}$.
\begin{proposition}
    For any $\delta \in (0,1)$ it holds $\cE^{\opt} \subseteq \cE^{\anticonc}(\delta)$.
\end{proposition}
\begin{proof}
    We proceed by a backward induction over $h$. Base of induction $h = H+1$ is trivial. Next by Bellman equations for $\uQ^t_h$ and $\Qstar_h$
    \[
        [\uQ^{t}_h - \Qstar_h](s,a) = \max_{j \in [J]} \left\{ \tp^{\,t,j}_h \uV^{t}_{h+1}(s,a)\right\} - p_h \Vstar_{h+1}(s,a).
    \]
    By induction hypothesis we have $\uV^{t}_{h+1}(s') \geq \uQ^{t}_{h+1}(s', \pi^\star(s')) \geq \Qstar_{h+1}(s', \pi^\star(s')) = \Vstar_{h+1}(s')$, thus
    \[
        [\uQ^{t}_h - \Qstar_h](s,a) \geq \max_{j \in [J]} \left\{ \tp^{\,t,j}_h \Vstar_{h+1}(s,a)\right\} - p_h \Vstar_{h+1}(s,a).
    \]
    By the definition of event $\cE^{\anticonc}(\delta)$ we conclude the statement.
\end{proof}

\begin{proposition}[Optimism]\label{prop:optimism}
    Assume that $J = \lceil c_J \cdot \log(2SAHT/\delta)  \rceil$, $\kappa = 2\beta^\star(\delta, T)$, $\ur = 2$ and $n_0 = \lceil (c_0 + \log_{17/16}(T/\kappa)) \cdot \kappa \rceil$, where $c_0$ is defined in \eqref{eq:constant_c0} and $c_J$ is defined in \eqref{eq:constant_cJ}. Then $\PP{\cE^{\opt} \mid  \cE^\star(\delta)} \geq 1-\delta/2$.
\end{proposition}

\subsection{Proof of Theorem~\ref{th:regret_bound_OPSRL}}
\label{app:regret}

First, we define an event $\cG(\delta) = \cG^{\conc}(\delta) \cap \cE^{\opt}$ where $\cG^{\conc}$ defined in Lemma~\ref{lem:proba_master_event}, and $\cE^{\opt}$ defined in \eqref{eq:opt_event}. This event is handle all required concentration and anti-concentration bounds for the proof of the regret bound. Lemma~\ref{lem:proba_master_event} and Proposition~\ref{prop:optimism} yield the following
\begin{corollary}\label{cor:good_event}
    Let conditions of Lemma~\ref{lem:proba_master_event} and Proposition~\ref{prop:optimism} hold. Then \( \PP{\cG(\delta)} \geq 1-\delta\).
\end{corollary}

Next, denote $\delta^t_h \triangleq \uV^{t-1}_h(s^t_h) - V^{\pi^t}_h(s^t_h)$ and surrogate regret $\uregret^{\,t}_h \triangleq \sum_{t=1}^T \delta^t_h$. To simplify notations denote $N^{\,t}_h = n^{\,t-1}_h(s^t_h, a^t_h)$, $N^{\,t}_h(s) =n^{\,t-1}_h(s | s^t_h, a^t_h), \upN^{\,t}_h = \upn^{\,t-1}_h(s^t_h, a^t_h), \upN^{\,t}_h(s) = \upn^{\,t-1}_h(s | s^t_h, a^t_h)$. Let
\begin{equation}
\label{eq: L delta def}
    L = \max\left\{ n_0 / \kappa, \log(TH), \betastar(\delta, T),  \beta^{\KL}(\delta, T), \beta^{\Dir}(\delta, T, J), \beta^{\conc}(\delta, T), \beta(\delta), \beta^{\Var}(\delta, T)\right\}.
\end{equation}
Under conditions Lemma~\ref{lem:proba_master_event} and Proposition~\ref{prop:optimism}, $L = \cO(\log(SATH/\delta)) = \tcO(1)$, $n_0 \leq 2L^2 = \cO(\log^2(SATH/\delta))$, and $\kappa \leq 2L$. In what follows we will follow ideas of \UCBVI with the Bernstein bonuses, see \citet{azar2017minimax}, and \BayesUCBVI, see \citet{tiapkin22dirichlet}.

\begin{lemma}\label{lem:surrogate_regret_bound}
    Assume conditions of Theorem \ref{th:regret_bound_OPSRL}. Then it holds on the event $\cG(\delta)$, for any $h \in [H],$
    \[
        \uregret^T_{h} \leq U^T_{h} \triangleq A^T_{h} + B^T_{h} + C^T_{h} + 4\rme H\sqrt{2 H T L} + 2\rme H^2 SA,
    \]
    where 
    \begin{align*}
        A^T_{h} &= 3\rme L \sum_{t=1}^T \sum_{h'=h}^H \sqrt{\Var_{\up^{\,t-1}_{h'}}[\uV^{t-1}_{h+1}](s^t_{h'},a^t_{h'}) \cdot \frac{\ind\{ N^{t}_{h'} > 0\}}{N^{t}_{h'}}} \CommaBin \\
        B^T_{h} &= \rme \sqrt{2L}\sum_{t=1}^T \sum_{h'=h}^H\sqrt{\Var_{p_{h'}}[\Vstar_{h+1}](s^t_{h'},a^t_{h'}) \frac{\ind\{ N^{t}_{h'} > 0\}}{N_{h'}^t}}\CommaBin \\
        C^T_{h} &= 26 H^2 S L^2 \sum_{t=1}^T \sum_{h'=h}^H \frac{\ind\{ N^{t}_{h'} > 0\}}{N^t_{h'}}\CommaBin
    \end{align*}
    and $L$ is defined in \eqref{eq: L delta def}.
\end{lemma}
\begin{proof}
    Since our actions are greedy with respect to $\uQ^{t-1}_h$, we have $\uV^{t-1}_h(s^t_h) = \uQ^{t-1}_h(s^t_h, a^t_h)$. Then by Bellman equations for $V^{\pi^t}$ and $\uQ^{t-1}_h$
    \begin{align*}
        \delta^t_h &= \uQ^{t-1}_h(s^t_h, a^t_h) - Q^{\pi^t}_h(s^t_h, a^t_h) =  \max_{j \in [J]}\left\{ \tp^{\,t-1,j}_h \uV^{t-1}_{h+1}(s^t_h, a^t_h) \right\} - p_h V^{\pi^t}_{h+1}(s^t_h, a^t_h) \\
        &= \underbrace{\max_{j \in [J]}\left\{ \tp^{\,t-1,j}_h \uV^{t-1}_{h+1}(s^t_h, a^t_h) \right\} - \up^{\,t-1}_h \uV^{t-1}_{h+1}(s^t_h, a^t_h)}_{\termA} + \underbrace{[\up^{\,t-1}_h - \hp^{\,t-1}_h] \uV^{t-1}_{h+1}(s^t_h, a^t_h)}_{\termB} \\
        &+ \underbrace{[\hp^{\,t-1}_h - p_h] [\uV^{t-1}_{h+1} - \Vstar_{h+1}] (s^t_h, a^t_h)}_{\termC} +  \underbrace{[\hp^{\,t-1}_h - p_h] \Vstar_{h+1}(s^t_h, a^t_h)}_{\termD} \\
        & + \underbrace{p_h[\uV^{t-1}_{h+1} - V^{\pi^t}_{h+1}](s^t_h, a^t_h) - [\uV^{t-1}_{h+1} - V^{\pi^t}_{h+1}](s^t_{h+1})}_{\xi^t_h} + \delta^t_{h+1}.
    \end{align*}
    
    This decomposition is almost equivalent to the decomposition in the proof of \BayesUCBVI, the main difference is $\termA$ and an another value of $n_0$. We notice that the term $\xi^t_h$ is an exactly the term that appears in the definition of the event $\cE(\delta) \subseteq \cG^{\conc}(\delta)$ in Lemma~\ref{lem:proba_master_event}.
    
    Let us analyze each term in this representation under assumption $N^t_h > 0$.
    \paragraph{Term $\termA$.}
    To handle this term, we use the event $\cE^{\Dir}(\delta) \subseteq \cG^{\conc}(\delta)$
    \begin{align*}
        \max_{j \in [J]}\left\{ \tp^{\,t-1,j}_h \uV^{t-1}_{h+1}(s^t_h, a^t_h) \right\} - \up^{\,t-1}_h \uV^{t-1}_{h+1}(s^t_h, a^t_h) &\leq 2 \sqrt{  \Var_{\up^{\,t-1}_h}[\uV^{t-1}_{h+1}](s^t_h,a^t_h) \frac{2L^2}{\upN^{\,t}_h}} + 3r_0H\frac{2L^2}{\upN^{\,t}_h} \\
        &\leq 3L \sqrt{ \frac{\Var_{\up^{\,t-1}_h}[\uV^{t-1}_{h+1}](s^t_h,a^t_h)}{\upN^{\,t}_h}} + \frac{12 HL^2}{\upN^{\,t}_h}\cdot
    \end{align*}
    
    \paragraph{Term $\termB$.}
    To bound $\termB$ we use directly a definition of $\up^{\,t-1}_h$ and $\hp^{\,t-1}_h$
    \[
        [\up^{\,t-1}_h - \hp^{\,t-1}_h] \uV^{t-1}_{h+1}(s^t_h, a^t_h) \leq \frac{n_0 r_0 H}{\upN^t_h} \leq \frac{4HL^2}{\upN^t_h}\cdot
    \]
    
    \paragraph{Term $\termC$.}
    Note that by Corollary~\ref{cor:good_event} the event $\cE^{\opt}$ holds. We see that $[\uV^{t-1}_{h+1} - \Vstar_{h+1}]$ is a non-negative function and therefore Lemma~\ref{lem:f_and_l1_concentration} is applicable for $f(s') = [\uV^{t-1}_{h+1} - \Vstar_{h+1}](s')$
    \begin{align*}
        [\hp^{\,t-1}_h - p_h] [\uV^{t-1}_{h+1} - \Vstar_{h+1}] (s^t_h, a^t_h) &\leq  \frac{1}{H} p_h [\uV^{t-1}_{h+1} - \Vstar_{h+1}](s^t_h, a^t_h) +  \frac{5 \ur H^2 S \cdot \beta^{\KL}(\delta, N^{\,t}_h)}{N^{\,t}_h} \\
        &\leq \frac{1}{H} (\xi^{\,t}_h + \delta^t_h) + \frac{10 H^2S \cdot L}{N^{\,t}_h}.
    \end{align*}
    
    \paragraph{Term $\termD$.}
    The bound on this term is guaranteed by the event $\cE^{\conc}(\delta) \subseteq \cG^{\conc}(\delta)$
    \begin{align*}
        (\hp^{\,t-1}_h - p_h) \Vstar_{h+1}(s^t_h, a^t_h) \leq \sqrt{2 \Var_{p_h}[\Vstar_{h+1}](s^t_h,a^t_h)\frac{L}{N_h^t}} + \frac{3HL}{N_h^t}\cdot
    \end{align*}
    All bounds on $\termA-\termD$ yield for $N^t_h > 0$
    \begin{align*}
        \delta^t_h &\leq \left(1 + \frac{1}{H} \right) \delta^t_h + \left(1 + \frac{1}{H} \right) \xi^t_h \\
        & + 3L\sqrt{\frac{\Var_{\up^{\,t-1}_h}[\uV^{t-1}_{h+1}](s^t_h,a^t_h)}{\upN^{\,t}_h}} + \sqrt{2L} \cdot \sqrt{\frac{\Var_{p_h}[\Vstar_{h+1}](s^t_h,a^t_h)}{N_h^t}}\\
        &+ \frac{10 H^2S \cdot L}{N^{\,t}_h} + \frac{16 L^2 H}{\upN^t_h}\cdot
    \end{align*}
    Additionally, there is a trivial bound $\delta^t_h \leq 2H$ that is valid for the case $N^t_h = 0$. 
    
    Define $\gamma_{h'} = (1 + 1/H)^{H -h'+1}$ and notice that $\gamma_{h'} \leq \rme, \upN^t_h \geq N^t_h$. Summing it up in the definition of $\uregret^{\,T}_h$ we obtain
    \begin{align*}
        \uregret^{\,T}_h &\leq \sum_{t=1}^T \sum_{h'=h}^H \gamma_{h'} \xi^t_{h'} + \sum_{t=1}^T \sum_{h'=h}^H 2 \rme H \ind\{ N^{t}_{h'} = 0\}\\ 
        &+ 3\rme L \sum_{t=1}^T \sum_{h'=h}^H \sqrt{\Var_{\up^{\,t-1}_{h'}}[\uV^{t-1}_{h+1}](s^t_{h'},a^t_{h'}) \cdot \frac{\ind\{N^{t}_{h'} > 0\}}{N^{t}_{h'}}} & \triangleq A^T_h \\ 
        &+ \rme \sqrt{2L}\sum_{t=1}^T \sum_{h'=h}^H\sqrt{\Var_{p_{h'}}[\Vstar_{h'+1}](s^t_{h'},a^t_{h'}) \frac{\ind\{ N^{t}_{h'} > 0\}}{N_{h'}^t}} &\triangleq B^T_h \\
        &+26 H^2 S L^2 \sum_{t=1}^T \sum_{h'=h}^H \frac{\ind\{ N^{t}_{h'} > 0\}}{N^t_{h'}} &\triangleq C^T_{h}.
    \end{align*}
    
    The bound on the first term is this decomposition follows from the definition of the event $\cE(\delta) \subseteq \cG^{\conc}(\delta) \subseteq \cG(\delta)$. To bound the second term we notice that the event $\ind\{ N^{t}_{h'} = 0\}$ could occur no more than $SAH$ times.
\end{proof}

Next we provide two important technical results. First of them is a classical result that follows from the pigeonhole principle.
\begin{lemma}\label{lem:regret_sums}
    For any $H,T \geq 1$,
    \begin{align*}
        \sum_{t=1}^T \sum_{h=1}^H &\frac{\ind\{n^{\,t-1}_h(s^t_h, a^t_h) > 0\}}{n^{\,t-1}_h(s^t_h, a^t_h)} \leq 2 HSAL, \\
        \sum_{t=1}^T \sum_{h=1}^H &\frac{\ind\{n^{\,t-1}_h(s^t_h, a^t_h) > 0\}}{\sqrt{n^{\,t-1}_h(s^t_h, a^t_h))}} \leq 3H \sqrt{TSA}.
    \end{align*}
\end{lemma}
\begin{proof}
     The main observation for both inequalities follows from pigeon-hole principle: term corresponding to each state-action pair $(s,a)$ appears in the sum exactly $n^{\,t-1}_h(s,a)$ times with a value $1/n$ for $n$ increasing from $1$, thanks to the indicator, to $n^{\,t-1}_h(s,a)$. For the first sum we use a bound on harmonic numbers, for the second one the integral bound.
\end{proof}

\begin{lemma}\label{lem:sum_variance}
    Assume that conditions of Theorem~\ref{th:regret_bound_OPSRL} are fulfilled. Then it holds on the event $\cG(\delta)$,
    \begin{align*}
        \sum_{t=1}^T \sum_{h=1}^H  \Var_{\up^{\,t-1}_h}[\uV^{t-1}_{h+1}](s^t_h,a^t_h) \ind\{N^t_h > 0\} &\leq 2 H^2 T + 2 H^2 U_1^T  + 38H^3S^2A L^3 + 32 H^3 S\sqrt{2AT L},\\
        \sum_{t=1}^T \sum_{h=1}^H  \Var_{p_h}[\Vstar_{h+1}](s^t_h,a^t_h) &\leq 2H^2T + 2H^2U_1^T + 6 H^3 L +  8\sqrt{2H^5 T L}.
    \end{align*}
    where $U_h^T$ is defined in Lemma~\ref{lem:surrogate_regret_bound}.
\end{lemma}
\begin{proof}
    First, apply the second inequality in Lemma~\ref{lem:switch_variance},
    \begin{align*}
        \sum_{t=1}^T \sum_{h=1}^H  \Var_{\up^{\,t-1}_h}[\uV^{t-1}_{h+1}](s^t_h,a^t_h) \ind\{N^t_h > 0\} &\leq  \underbrace{\sum_{t=1}^T \sum_{h=1}^H \Var_{p_ h}[\uV^{t-1}_{h+1}](s^t_h,a^t_h) \ind\{N^t_h > 0\}}_{\termOne} \\
        &+ \underbrace{2 \ur^2 H^2 \sum_{t=1}^T \sum_{h=1}^H \norm{ \up^{\,t-1}_h(s^t_h, a^t_h) -  p_h(s^t_h, a^t_h) }_1 \ind\{N^t_h > 0\}}_{\termTwo}.
    \end{align*}

    To bound the term $\termTwo$ one may use Lemma~\ref{lem:f_and_l1_concentration}. We obtain under assumption $N^t_h > 0$
    \begin{align*}
        \norm{ \up^{\,t-1}_h(s^t_h, a^t_h) -  p_h(s^t_h, a^t_h) }_1 &\leq \norm{\up^{\,t-1}_h(s^t_h, a^t_h) - \hp^{\,t-1}_h(s^t_h, a^t_h)}_1 + \norm{p_h(s^t_h, a^t_h) - \hp^{\,t-1}_h(s^t_h, a^t_h)}_1  \\
        &\leq \frac{n_0}{\upN^{\,t}_h} + \sum_{s \in \cS} N^{\,t}_h(s) \left(\frac{1}{N^{\,t}_h} - \frac{1}{\upN^{\,t}_h}\right) + \sqrt{\frac{2S L}{N^{\,t}_h}} \leq \frac{2SL^2}{N^{\,t}_h} + \sqrt{\frac{2SL}{N^{\,t}_h}}.
    \end{align*}
    Since $\ur = 2$, Lemma~\ref{lem:regret_sums} implies
    \[
        \termTwo \leq 2 \ur^2 H^2 \sum_{t=1}^T \sum_{h=1}^H \norm{ \up^{\,t-1}_h(s^t_h, a^t_h) -  p_h(s^t_h, a^t_h) }_1 \ind\{N^t_h > 0\} \leq 32H^3S^2A L^3 + 24 H^3 S\sqrt{2 AT L}.
    \]
    Next, we bound $\termOne$ using the first inequality of Lemma~\ref{lem:switch_variance}
    \[
        \termOne \leq 2 \underbrace{\sum_{t=1}^T \sum_{h=1}^H \Var_{p_ h}[V^{\pi^t}_{h+1}](s^t_h,a^t_h)}_{\termThree} + 2\underbrace{\sum_{t=1}^T \sum_{h=1}^H  \ur H p_h\left| \uV^{t-1}_{h+1} - V^{\pi^t}_{h+1} \right|(s^t_h, a^t_h)}_{\termFour}.
    \]
    The term $\termThree$ could be bounded using definition of the event $\cE^{\Var}(\delta)$. It follows that
    \[
        \termThree  \leq H^2 T + \sqrt{2H^5 T L} + 3H^3 L.
    \]
    By Proposition~\ref{prop:optimism} we have $\uV^{t-1}_{h+1}(s) \geq V^{\pi^t}_{h+1}(s)$ for any $s \in \cS$. By the definition of $\xi^{\,t}_h, \delta^t_h$, definition of event $\cE(\delta)$ term, and Lemma~\ref{lem:surrogate_regret_bound} $\termFour$ could be bounded as follows
    \begin{align*}
        \termFour &\leq \sum_{t=1}^T \sum_{h=1}^H 2H(\xi^{\,t}_h + \delta^t_{h+1}) \leq 2\ur H^2\sqrt{2T L} + 2H \sum_{h=1}^H \uregret^T_{h+1} \leq 4H^2 \sqrt{2T L} + 2H^2 U^T_1.
    \end{align*}
    Therefore, we have
    \begin{align*}
        \sum_{t=1}^T \sum_{h=1}^H  \Var_{\up^{\,t-1}_h}[\uV^{t-1}_{h+1}](s^t_h,a^t_h) &\leq \termTwo + 2 \cdot \termThree + 2 \cdot \termFour \\
        &\leq 2 H^2 T + 2 H^2 U_1^T  + (32 + 6)H^3S^2A L^3 + (24+8) H^3 S\sqrt{2AT L}  \\
        &\leq 2 H^2 T + 2 H^2 U_1^T  + 38 H^3S^2A L^3 + 32 H^3 S\sqrt{2 AT L}.
    \end{align*}
    The first inequality in Lemma~\ref{lem:switch_variance} gives us a bound for the second inequality
    \[
        \sum_{t=1}^T \sum_{h=1}^H  \Var_{p_h}[\Vstar_{h+1}](s^t_h,a^t_h) \leq 2 \underbrace{\sum_{t=1}^T \sum_{h=1}^H \Var_{p_ h}[V^{\pi^t}_{h+1}](s^t_h,a^t_h)}_{\termThree} + 2\sum_{t=1}^T \sum_{h=1}^H  \ur H p_h\left| \Vstar_{h+1} - V^{\pi^t}_{h+1} \right|(s^t_h, a^t_h).
    \]
    
    Since $\uV^{t-1}_h \geq \Vstar_h$ by Proposition~\ref{prop:optimism}, the second term is bounded by  $\termFour$. Thus
    \begin{align*}
        \sum_{t=1}^T \sum_{h=1}^H  \Var_{p_h}[\Vstar_{h+1}](s^t_h,a^t_h)  &\leq 2\termThree + 2 \termFour  \leq 2H^2T + 2H^2U_1^T + 8\sqrt{2H^5 T L} + 6 H^3 L.
    \end{align*}
\end{proof}

\begin{lemma}\label{lem:surrogate_regret_terms_bound}
    Assume conditions of Theorem~\ref{th:regret_bound_OPSRL} and Lemma~\ref{lem:surrogate_regret_bound}. Then on the event $\cG(\delta)$ it holds
    \begin{align*}
        A^T_1 &\leq 6\rme\sqrt{H^3 SAT} \cdot  L^{3/2} + 6\rme \sqrt{H^3 SA U^T_1} \cdot L^{3/2} + 27\rme  H^2 S^{3/2} A L^{3} + 30 \rme H^2 S A^{3/4} T^{1/4} L^{7/4}, \\
        B^T_1 &\leq 4\rme \sqrt{H^3 SAT} \cdot  L + 4 \rme\sqrt{H^3 SA U^T_1} \cdot  L + 8 \rme H^2 S^{1/2} A^{1/2} L^2  + 10 \rme H^{7/4} S^{1/2} A^{1/2} T^{1/4} L^{5/4},\\
        C^T_1 &\leq 52 \rme H^3 S^2 A L^3 = \tcO(H^3 S^2 A).
    \end{align*}
\end{lemma}
\begin{proof}
    For the term $A^T_1$ we apply the Cauchy–-Schwartz inequality, Lemma~\ref{lem:sum_variance}, Lemma~\ref{lem:regret_sums} and inequality $\sqrt{a+b} \leq \sqrt{a} + \sqrt{b}, a, b \geq 0$, 
    \begin{align*}
        \sum_{t=1}^T \sum_{h=1}^H &\sqrt{  \Var_{\up^{\,t-1}_{h}}[\uV^{t-1}_{h+1}](s^t_{h},a^t_{h}) \frac{\ind\{N^t_h > 0\}}{N^t_{h}}} \\
        &\leq \sqrt{\sum_{t=1}^T \sum_{h=1}^H \Var_{\up^{\,t-1}_{h}}[\uV^{t-1}_{h+1}](s^t_{h},a^t_{h}) \ind\{N^t_h > 0\}} \cdot \sqrt{\sum_{t=1}^T \sum_{h=1}^H \frac{\ind\{N^t_h > 0\}}{N^t_{h}}} \\
        &\leq \sqrt{ 2 H^2 T + 2 H^2 U_1^T  + 38 H^3S^2A L^3 + 32 H^3 S\sqrt{2 AT L}} \cdot \sqrt{2SAH L} \\
        &\leq 2\sqrt{H^3 SAT  L} + 2\sqrt{H^3 SA U^T_1 L} + 9 H^2 S^{3/2} A L^{2} + 10 H^2 S A^{3/4} T^{1/4} L^{3/4}.
    \end{align*}
    Similarly, the term $B^T_1$ may be estimated as follows
    \begin{align*}
        \sum_{t=1}^T \sum_{h=1}^H& \sqrt{\Var_{p_{h}}[\Vstar_{h+1}](s^t_{h},a^t_{h})\frac{\ind\{N^t_h > 0\}}{N_{h}^t}} \leq \sqrt{\sum_{t=1}^T \sum_{h=1}^H \Var_{p_{h}}[\Vstar_{h+1}](s^t_{h},a^t_{h})} \cdot \sqrt{\sum_{t=1}^T \sum_{h=1}^H \frac{\ind\{N^t_h > 0\}}{N^t_{h}}} \\
        &\leq \sqrt{2H^2T + 2H^2U_1^T + 8\sqrt{2H^5 T L} + 6 H^3 L} \cdot \sqrt{2SAH \cdot L} \\
        &\leq 2\sqrt{H^3 SAT  L} + 2\sqrt{H^3 SA U^T_1 L} + 4 H^2 L \sqrt{SA} + 5 H^{7/4}T^{1/4} L^{3/4}\sqrt{SA}.
    \end{align*}
    Finally, to estimate $C^T_1$ we apply Lemma~\ref{lem:regret_sums}. We obtain
    \[
        C^T_1 \leq 26\rme H^2 S \cdot L^2 \cdot 2SAH L \leq 52 \rme H^3 S^2 A L^3.
    \]
\end{proof}

\begin{proof}[Proof of Theorem~\ref{th:regret_bound_OPSRL}]
    By Corollary~\ref{cor:good_event} event $\cG(\delta)$ holds with probability at least $1-\delta$. Next we assume that this event holds. Then we have two cases: $T < H^2 S^2 A L^3$ and $T \geq H^2 S^2 A L^3$. In the first case the regret is trivially bounded by $\regret^T \leq H^3 S^2 A L^3$. Thus it is sufficient to analyze only the second case.
    
    By Proposition~\ref{prop:optimism} and Lemma~\ref{lem:surrogate_regret_bound}
    \begin{equation}
    \label{eq: regret estimation0}
    \begin{split}
        \regret^T &= \sum_{t=1}^T \Vstar_h(s^t_1) - V^{\pi^t}_h(s^t_1) \leq \sum_{t=1}^T \uV^{\,t-1}_h(s^t_1) - V^{\pi^t}_h(s^t_1) = \uregret^T_1 \\
        &\leq U^T_1 = A^T_1 + B^T_1 + C^T_1 + 4\rme \sqrt{2 H^{3} T L} + 2\rme H^2 SA.
    \end{split}
    \end{equation}
    Next, under our condition on $T$ we can simplify expressions for the bounds of $A^T_1$ and $B^T_1$. Indeed, $T  \geq H^2 S^2 A L^3$ implies that
    \[
     H^{7/4} S^{1/2} A^{1/2} L^{5/4} \cdot T^{1/4}\leq  H^2 S A^{3/4} L^{7/4}\cdot T^{1/4} \leq \sqrt{H^3 S A  T} L^{3/2}.
    \]
    Furthermore,
    \[
        H^2 S^{3/2} A L^3 \leq H^3 S^2 A L^3,\qquad  H^2 S^{1/2} A^{1/2} L^2 \leq H^3 S^2 A L^3, \qquad \sqrt{2H^{3} T L} \leq \sqrt{2H^3 SAT} \cdot L.
    \]
    We obtain the following bounds
    \begin{align*}
        A^T_1 &\leq 36\rme \sqrt{H^3 SAT} \cdot  L^{3/2} + 6\rme \sqrt{H^3 SA U^T_1} \cdot L^{3/2} + 27 \rme H^3 S^2 A L^3, \\
        B^T_1 &\leq 14\rme \sqrt{H^3 SAT} \cdot  L + 4 \rme\sqrt{H^3 SA U^T_1} \cdot  L + 8 \rme H^3 S^2 A L^3,\\
        C^T_1 &\leq 52 \rme H^3 S^2 A L^3.
    \end{align*}
    Hence, using bound $H^2 S A \leq H^3 S^2 A$,
    \begin{align*}
        U^T_1 &\leq 50 \rme \sqrt{H^3 SAT} \cdot L^{3/2} + 10\rme \sqrt{H^3 SA U^T_1} \cdot L^{3/2} + 89\rme H^3 S^2 A L^{3} + 4\rme\sqrt{2} \cdot \sqrt{H^3 T L} \\
        &\leq 56 \rme \sqrt{H^3 SAT} \cdot L^{3/2} + 10\rme \sqrt{H^3 SA U^T_1} \cdot L^{3/2} + 89\rme H^3 S^2 A L^{3}.
    \end{align*}
    This is a quadratic inequality in $U^T_1$. Solving this inequality and using inequality $\sqrt{a+b} \leq \sqrt{a} + \sqrt{b}, a, b \geq 0$, we obtain
    \begin{align*}
         U^T_1 \leq 108\rme\sqrt{H^3 S A T L^3} + 178 \rme H^3 S^2 A L^3 +  200\rme^2 H^3 S A L^3. 
    \end{align*}
     The last inequality and \eqref{eq: regret estimation0} imply that 
    \[
        \regret^T  = \cO\left( \sqrt{H^3 SAT L^3} + H^3 S^2 A L^{3} \right).
    \]
    
\end{proof}

\newpage
\section{Deviation inequalities}
\label{app:deviation_ineq}
\subsection{Deviation inequality for categorical distributions}

Next, we restate the deviation inequality for categorical distributions by \citet[Proposition 1]{jonsson2020planning}.
Let $(X_t)_{t\in\N^\star}$ be i.i.d.\,samples from a distribution supported on $\{1,\ldots,m\}$, of probabilities given by $p\in\simplex_{m-1}$, where $\simplex_{m-1}$ is the probability simplex of dimension $m-1$. We denote by $\hp_n$ the empirical vector of probabilities, i.e., for all $k\in\{1,\ldots,m\},$
 \[
 \hp_{n,k} = \frac{1}{n} \sum_{\ell=1}^n \ind\left\{X_\ell = k\right\}.
 \]
 Note that  an element $p \in \simplex_{m-1}$ can be seen as an element of $\R^{m-1}$ since $p_m = 1- \sum_{k=1}^{m-1} p_k$. This will be clear from the context. 
 \begin{theorem} \label{th:max_ineq_categorical}
 For all $p\in\simplex_{m-1}$ and for all $\delta\in[0,1]$,
 \begin{align*}
     \P\left(\exists n\in \N^\star,\, n\KL(\hp_n, p)> \log(1/\delta) + (m-1)\log\left(e(1+n/(m-1))\right)\right)\leq \delta.
 \end{align*}
\end{theorem}

\subsection{Deviation inequality for categorical weighted sum}
 We fix a function $f: \{1,\ldots,m\} \to [0,b]$ and recall the definition of the minimum Kullback-Leibler divergence for $p\in\simplex_{m-1}$  and $u\in\R$
 \[
\Kinf(p,u,f) = \inf\left\{  \KL(p,q): q\in\simplex_{m-1}, qf \geq u\right\}\,.
 \]
As the Kullback-Leibler divergence this quantity admits a variational formula.
\begin{lemma}[Lemma 18 by \citealp{garivier2018kl}]
\label{lem:var_form_Kinf} For all $p \in \simplex_{m-1}$, $u\in [0,b)$,
\[
\Kinf(p,u,f) = \max_{\lambda \in[0,1]} \E_{X\sim p}\left[ \log\left( 1-\lambda \frac{f(X)-u}{b-u}\right)\right]\,,
 \]
 moreover if we denote by $\lambda^\star$ the value at which the above maximum is reached, then
 \[
   \E_{X\sim p} \left[\frac{1}{1-\lambda^\star\frac{f(X)-u}{b-u}}\right] \leq 1\,.
 \]
\end{lemma}
\begin{remark} Contrary to \citet{garivier2018kl} we allow that $u=0$ but in this case Lemma~\ref{lem:var_form_Kinf} is trivially true, indeed
  \[
  \Kinf(p,0,f) =  0  = \max_{\lambda \in[0,1]} \E_{X\sim p}\left[ \log\left( 1-\lambda \frac{f(X)}{b}\right)\right]\,.
   \]
\end{remark}

We are now ready to restate the deviation inequality for the $\Kinf$ by \citet{tiapkin22dirichlet} which is a self-normalized version of Proposition~13 by \citet{garivier2018kl}.
 \begin{theorem} \label{th:max_ineq_kinf}
 For all $p\in\simplex_{m-1}$ and for all $\delta\in[0,1]$,
 \begin{align*}
     \P\big(\exists n\in \N^\star,\, n\Kinf(\hp_n, pf, f)> \log(1/\delta) + 3\log(e\pi(1+2n))\big)\leq \delta.
 \end{align*}
\end{theorem}

\subsection{Deviation inequality for bounded distributions}
Below, we restate the self-normalized Bernstein-type inequality by \citet{domingues2020regret}. Let $(Y_t)_{t\in\N^\star}$, $(w_t)_{t\in\N^\star}$ be two sequences of random variables adapted to a filtration $(\cF_t)_{t\in\N}$. We assume that the weights are in the unit interval $w_t\in[0,1]$ and predictable, i.e. $\cF_{t-1}$ measurable. We also assume that the random variables $Y_t$  are bounded $|Y_t|\leq b$ and centered $\EEc{Y_t}{\cF_{t-1}} = 0$.
Consider the following quantities
\begin{align*}
		S_t \triangleq \sum_{s=1}^t w_s Y_s, \quad V_t \triangleq \sum_{s=1}^t w_s^2\cdot\EEc{Y_s^2}{\cF_{s-1}}, \quad \mbox{and} \quad W_t \triangleq \sum_{s=1}^t w_s
\end{align*}
and let $h(x) \triangleq (x+1) \log(x+1)-x$ be the Cramér transform of a Poisson distribution of parameter~1.

\begin{theorem}[Bernstein-type concentration inequality]
  \label{th:bernstein}
	For all $\delta >0$,
	\begin{align*}
		\PP{\exists t\geq 1,   (V_t/b^2+1)h\left(\!\frac{b |S_t|}{V_t+b^2}\right) \geq \log(1/\delta) + \log\left(4e(2t+1)\!\right)}\leq \delta.
	\end{align*}
  The previous inequality can be weakened to obtain a more explicit bound: if $b\geq 1$ with probability at least $1-\delta$, for all $t\geq 1$,
 \[
 |S_t|\leq \sqrt{2V_t \log\left(4e(2t+1)/\delta\right)}+ 3b\log\left(4e(2t+1)/\delta\right)\,.
 \]
\end{theorem}

\subsection{Deviation inequality for Dirichlet distribution}
Below we provide the Bernstein-type inequality for weighted sum of Dirichlet distribution, using a generalized result on upper bound on tails for linear statistics on Dirichlet distribution (Theorem~\ref{thm:upper_bound_dbc}). 

\begin{lemma}\label{lem:bernstein_dirichlet}[Generalization of Lemma C.8 by \cite{tiapkin22dirichlet}]
     For any $\alpha = (\alpha_0, \alpha_1, \ldots, \alpha_m) \in \R_{++}^{m+1}$ define  $\up \in \simplex_{m}$ such that $\up(\ell) = \alpha_\ell/\ualpha, \ell = 0, \ldots, m$, where $\ualpha = \sum_{j=0}^m \alpha_j$. Then for any $f \colon \{0,\ldots,m\} \to [0,b]$ such that $f(0) = b$ and $\delta \in (0,1)$
     \[
        \P_{w \sim \Dir(\alpha)}\left[wf \geq \up f +  2 \sqrt{ \frac{ \Var_{\up}(f) \log(1/\delta)}{\ualpha}} + \frac{3b \cdot \log(1/\delta)}{\ualpha} \right] \leq \delta.
     \]
\end{lemma}
\begin{remark}
  The only difference with the result of Lemma C.8 by \cite{tiapkin22dirichlet} is allowing to parameters of Dirichlet distribution being non-integer.
\end{remark}
\begin{proof}
   Fix $\delta \in (0,1)$ and let $\mu \in (\up f,b)$ be such that
\[
    \Kinf(\up, \mu, f) = \ualpha^{-1} \log (1/\delta). 
\]
Note that such $\mu$ exists by the continuity of $\Kinf$ w.r.t. the second argument, see \citet[Theorem 7]{honda2010asymptotically}. By \citet[Lemma C.7]{tiapkin22dirichlet} there exists $q$ such that $\up \ll q $, $q f = \mu$ and $\KL(\up, q) = \ualpha^{-1} \log (1/\delta)$. By Theorem~\ref{thm:upper_bound_dbc}   
\begin{equation}
\label{eq: bernstein 1}
 \P_{w \sim \Dir(\alpha)}[wf \geq qf] = \P_{w \sim \Dir(\alpha)}[wf \geq \mu] \leq \exp\left( -\ualpha \Kinf(\up, \mu, f) \right) = \delta.   
\end{equation}
By Lemma \ref{lem:Bernstein_via_kl}
\[
q f  - \up f \le \sqrt{2\Var_{q}(f)\KL(\up ,q)}.
\]
By Lemma \ref{lem:switch_variance_bis},  $\Var_q(f) \leq 2\Var_{\up}(f) +4b^2 \KL(\up ,q)$.
The last two inequalities and \eqref{eq: bernstein 1} imply that
\begin{equation*}
  \P_{w \sim \Dir(\alpha)}\left[wf - \up f \geq \sqrt{4 \Var_{\up}(f) \KL(\up ,q)} + 2 b \sqrt 2 \cdot      \KL(\up,q) \right] \le \delta. 
\end{equation*}
\end{proof}

\newpage
\section{Exponential and Gaussian approximations of Dirichlet distribution}
\label{app:gaussian_approximation}

In this section we present result on approximation of a tail probabilities for linear statistics of Dirichlet distribution from above by tails of exponential distribution and from below by tails of Gaussian distribution.

The proof of upper bound generalizes proof of \cite{baudry2021optimality} to non-integer parameters using exactly the same technique; see also \cite{riou20a}.
\begin{theorem}[Upper bound]\label{thm:upper_bound_dbc}
    For any $\alpha = (\alpha_0, \alpha_1, \ldots, \alpha_m) \in \R_{++}^{m+1}$ define  $\up \in \simplex_{m}$ such that $\up(\ell) = \alpha_\ell/\ualpha, \ell = 0, \ldots, m$, where $\ualpha = \sum_{j=0}^m \alpha_j$. Then for any $f \colon \{0,\ldots,m\} \to [0,b]$ and $0 < \mu < b$ and
  \[
        \P_{w\sim \Dir(\alpha)} [wf \geq \mu] \leq \exp(-\ualpha \Kinf(\up,\mu, f)).
  \]
\end{theorem}
\begin{proof}
    The statement is trivial for $\mu \leq \up f$ since $\Kinf(\up,\mu, f) = 0$. Assume that $\mu > \up f$. It is well know fact that $w \sim \Dir(\alpha)$ may be represented as follows
  \[
        w \triangleq \left( \frac{Y_0}{V_m}, \frac{Y_1}{V_m}, \ldots, \frac{Y_m}{V_m} \right),
  \]
  where $Y_\ell \mysim \Gamma(\alpha_\ell, 1), \ell = 0, \ldots, m$ and $V_m = \sum_{\ell=0}^m Y_\ell$. Let us fix $\lambda \in [0, 1/(b - u))$ and proceed by the changing of measure argument
  \begin{align*}
        \PP{wf \geq \mu} &= \PP{ \sum_{\ell=0}^m Y_\ell f(\ell) \geq \sum_{\ell = 0}^m Y_\ell \mu } = \E_{Y_\ell \mysim \Gamma(\alpha_\ell, 1)}\left[ \ind\left\{ \sum_{\ell=0}^m Y_\ell(f(\ell) - \mu) \geq 0 \right\} \right] \\
        &= \E_{\hat{Y}_\ell \mysim \Gamma(\alpha_\ell, 1 - \lambda(f(\ell) - \mu))}\left[ \ind\left\{ \sum_{\ell=0}^m \hat{Y}_\ell(f(\ell) - \mu) \geq 0 \right\} \prod_{\ell=0}^m \frac{\exp(-\lambda \hat{Y}_{\ell} (f(\ell) - \mu))}{(1 - \lambda(f(\ell) - \mu))^{\alpha_\ell}} \right] \\
        &= \exp\left( - \sum_{\ell=0}^m \alpha_{\ell} \log(1 - \lambda(f(\ell) - \mu)) \right)  \\
        &\qquad \cdot \E_{\hat{Y}_\ell \mysim \Gamma(\alpha_\ell, 1 - \lambda(f(\ell) - \mu))}\left[ \ind\left\{ \sum_{\ell=0}^m \hat{Y}_\ell(f(\ell) - \mu) \geq 0 \right\} \rme^{-\lambda \sum_{\ell=0}^m \hat{Y}_{\ell} (f(\ell) - \mu)}\right] \\
        &\leq \exp\left( -\ualpha \sum_{\ell=0}^m \up_{\ell} \log(1 - \lambda(f(\ell) - \mu)) \right) = \exp\left( -\ualpha \E_{X \sim \up}\left[ \log(1 - \lambda(f(X) - \mu) \right] \right).
  \end{align*}
  
   Since the previous inequality is true for all $\lambda \in[0,1/(b-\mu))$, then the variational formula (Lemma~\ref{lem:var_form_Kinf}) allows to conclude
   \begin{align*}  
        \P_{w\sim \Dir(\alpha)} [w f\geq \mu] &\leq  \exp\left(-\ualpha\sup_{\lambda \in[0,1/(b-\mu))}\E_{X \sim \up}\left[ \log(1 - \lambda(f(X) - \mu) \right]\right) \\
        &= \exp(-\ualpha \Kinf(\up,\mu, f)).
    \end{align*}
\end{proof}

The proof of lower bound extends the approach of \cite{tiapkin22dirichlet} by ideas of \cite{alfers1984normal} and gives much more exact bounds. Define
\begin{equation}\label{eq:c0_eps}
    c_0(\varepsilon) = \left(\frac{4}{\sqrt{\log(17/16)}} + 8 + \frac{49 \cdot 4 \sqrt{6}}{9} \right)^2 \frac{2}{\pi \cdot \varepsilon^2} + \log_{17/16}\left( \frac{5}{32 \cdot \varepsilon^2}\right).
\end{equation}
\begin{theorem}[Lower bound]\label{thm:lower_bound_dbc} 
For any $\alpha = (\alpha_0+1, \alpha_1, \ldots, \alpha_m) \in \R_{++}^{m+1}$ define  $\up \in \simplex_{m}$ such that $\up(\ell) = \alpha_\ell/\ualpha, \ell = 0, \ldots, m$, where $\ualpha = \sum_{j=0}^m \alpha_j$. Let $\varepsilon \in (0,1)$. Assume that $\alpha_0 \geq c_0(\varepsilon) + \log_{17/16}(\ualpha)$ for $c_0(\varepsilon)$ defined in \eqref{eq:c0_eps}, and $\ualpha \geq 2\alpha_0$. Then for any $f \colon \{0,\ldots,m\} \to [0,\ub]$ such that $f(0) = \ub$, $f(j) \leq b < \ub/2, j \in \{1,\ldots,m\}$ and $\mu \in (\up f,  \ub)$ 
    \[
        \P_{w \sim \Dir(\alpha)}[wf \geq \mu] \geq (1 - \varepsilon)\P_{g \sim \cN(0,1)}\left[g \geq \sqrt{2 \ualpha \Kinf(\up, \mu, f)}\right].
    \]
\end{theorem}

In the further subsections we are going to prove this theorem.

\subsection{Proof of Theorem \ref{thm:lower_bound_dbc}} 

First, we restate the result of \cite{tiapkin22dirichlet} on representation of the density of linear statistic of Dirichlet distribution.
\begin{proposition}[Proposition D.3 of \cite{tiapkin22dirichlet}]\label{prop:dirichlet_density_integral}
    Let $f \in \Fclass_m(b)$ and $\alpha = (\alpha_0 + 1, \alpha_1, \ldots, \alpha_m) \in \R_{++}^{m+1}$ such that $\ualpha = \sum_{j=0}^m \alpha_j > 1$. Let $w \sim \Dir(\alpha)$ and assume that $Z = wf$ is not degenerate. Then for any $0 \leq u < \ub$
    \[
        p_{Z}(u) = \frac{\ualpha}{2\pi} \int_\R (1 + \rmi(\ub - u)s)^{-1} \prod_{j=0}^m \left( 1 + \rmi(f(j) - u)s \right)^{-\alpha_j} \rmd s.
    \]
\end{proposition}

Next we proceed in the same spirit as an approach of \cite{tiapkin22dirichlet} and apply the method of saddle point (see \cite{fedoryuk1977metod,olver1997asymptotics}) to derive an asymptotically tight approximation. However, in our case we have to extract one additional term in from of the product.

\begin{proposition}\label{prop:asymptotics_integral}
     Let $f \in \Fclass_m(\ub,b)$ and let  $\alpha = (\alpha_0+1, \alpha_1, \ldots, \alpha_m) \in \R_+^{m+1}$ be a fixed vector with $\alpha_0 \geq 2$. Then for any $u \in (\up f, \ub),$ 
    \[
        \int_\R \frac{\prod_{\ell=0}^m \left(1 + \rmi(f(\ell) - u)s \right)^{-\alpha_\ell}}{\left(1 + \rmi(\ub - u)s \right)} \rmd s = \left(\sqrt{\frac{2\pi}{\ualpha \, \sigma^2 }} - R_1(\alpha) + R_2(\alpha) \right) \frac{\exp(-\ualpha \,\Kinf(\up, u, f))}{1 - \lambda^\star(\ub - u)} + R_3(\alpha),
    \]
    where   
    \begin{align*}
        \sigma^2 &= \E_{X \sim \up}\left [\left(\frac{f(X) - u}{1 - \lambda^\star(f(X) - u)}\right)^2\right], 
                \\
        \vert R_1(\alpha) \vert &\leq \frac{c_1}{(1 - \lambda^\star(\ub - u))\sqrt{\sigma^2\, c_\kappa \alpha_0\, \ualpha}} , \\
        \vert R_2(\alpha) \vert &\leq   \frac{c_2}{(1 - \lambda^\star(\ub - u))\sqrt{\sigma^2\,\ualpha \,\alpha_0}}, \\
        \vert R_3(\alpha) \vert &\leq  c_3\cdot \frac{\exp(-\ualpha \Kinf(\up, u, f))}{1 - \lambda^\star(\ub - u)} \cdot \frac{1 - \lambda^\star(\ub-u)}{\ub-u}  \exp(-c_\kappa \alpha_0)
    \end{align*}
with $c_1 = 2\sqrt{2}, c_2 =\left(8 + \frac{49 \sqrt{6}}{9} \frac{\ub}{\ub - \up f} \right), c_3 = \frac{\sqrt{5\pi}}{2}, c_\kappa = 1/2 \cdot \log\left(1 + \frac{1}{4} \left(\frac{\ub - \up f}{\ub}\right)^2 \right)$ and $\lambda^\star$ being a solution to the optimization problem
    $$
        \lambda^\star(\up, u, f) = \argmax_{\lambda \in [0, 1/(\ub-u)]} \E_{X \sim \up}\left[\log(1 - \lambda (f(X) - u)) \right].
    $$
\end{proposition}
\begin{proof}
    We start from the rewriting the integral in the form that allows us to apply saddle point method,
    \begin{align}
        I &= \int_\R \frac{\prod_{j=0}^n \left(1 + \rmi (f(j) - u)s \right)^{-\alpha_j}}{1 - \rmi(\ub - u)s}\, \rmd s = \int_\R \frac{\exp\left( - \ualpha \sum_{j=0}^m \up_j \log(1 + \rmi (f(j) - u)s) \right)}{1 - \rmi(\ub - u)s}\, \rmd s  \notag \\
        &= \int_\R (1 - \rmi (\ub - u)s)^{-1} \exp\left( - \ualpha \, \E_{X \sim \up}[\log(1 + \rmi (f(X) - u)s)] \right) \, \rmd s. \label{eq:integral_representation}
    \end{align}
    Since the analysis of the suitable integration contour depends only on the function under exponent, we may directly switch to the contour $\gamma^\star = \rmi \lambda^\star + \R$ as it was stated in \cite{tiapkin22dirichlet}. 
    
    Next we continue following approach of \cite{tiapkin22dirichlet} and denote the following functions
    \begin{align*}
        T(s) &= \E\left[\log( 1 - \lambda^\star (f(X) - u) + \rmi s (f(X) - u))\right], \\
        P(s) &= \frac{1}{1 - \lambda^\star(\ub - u) + \rmi s (\ub - u)},
    \end{align*}
     a cut-off parameter $K > 0$,  and  define $\kappa_1 = T(-K) - T(0)$, $\kappa_2 = T(K) - T(0)$. Similarly to Chapter~4 (Section~6) by \citet{olver1997asymptotics}, we define the change of  variables $v_1 = T(-s) - T(0), v_2 = T(s) - T(0)$ and the implicit functions $q_1(v_1) = \frac{P(-s)}{T'(-s)}, \, q_2(v_2) = \frac{P(s)}{T'(s)}.$ Notice that these functions differs from ones defined in \cite{tiapkin22dirichlet} due to the presence of an additional multiplier $P(s)$. Using the first order  Taylor expansion, we can write $q_1(v_1) = \frac{P(0)}{\sqrt{2 T''(0) \cdot v_1}} + r_1(v_1), \, q_2(v_2) = \frac{P(0)}{\sqrt{2 T''(0) \cdot v_2}} + r_2(v_2)$. Then we have the following decomposition
    \[
        I = \int_{-\infty}^{\infty} P(s) \exp(-\ualpha\,  T(s))\, \rmd s = \left( P(0) \cdot \sqrt{\frac{2\pi}{\ualpha \, T''(0)}} - R_1(\alpha) + R_2(\alpha) \right) \exp(-\ualpha \, T(0)) + R_3(\alpha),
    \]
    where
    \begin{align*}
        R_1(\alpha) &=  \left(\Gamma\left( \frac{1}{2}, \kappa_1 \, \ualpha \right) + \Gamma\left( \frac{1}{2}, \kappa_2 \, \ualpha \right) \right) \frac{P(0)}{\sqrt{2 T''(0) \, \ualpha}}\CommaBin \\
        R_2(\alpha) &= \int_0^{\kappa_1} e^{-\ualpha v_1} r_1(v_1)\, \rmd v_1 + \int_0^{\kappa_2} e^{-\ualpha v_2} r_2(v_2)\, \rmd v_2, \\
        R_3(\alpha) &= \int_{\R \setminus [-K, K]} P(s) \exp(-\ualpha \, T(s))\, \rmd s,
    \end{align*}
    where $\Gamma(\alpha, x)$ is an upper incomplete gamma function and integration w.r.t. \(v_1,v_2\) is performed over the straight lines connecting the points \(0\) and \(\kappa_1,\kappa_2,\) respectively. Define $\sigma^2 = T''(0)$.
    
    \paragraph{Term $R_2$.}
    We will start from upper bounding on remainder terms in Taylor-like expansions $r_2(v)$
     \begin{align*}
        \vert r_2(v) \vert &= \left\vert \frac{P(s)}{T'(s)} - \frac{P(0)}{\sqrt{2 T''(0) (T(s) - T(0))}}\right\vert \\
        &\leq  P(0) \left\vert \frac{1}{T'(s)} - \frac{1}{\sqrt{2 T''(0) (T(s) - T(0))}}\right\vert  + \frac{\vert P(s) - P(0) \vert }{\vert T'(s) \vert} \\
        &= P(0) \cdot \bar r_2(v) + \tilde r_2(v).
    \end{align*}
    Analysis of the term $\bar r_2(v)$ was performed in \cite{tiapkin22dirichlet} under the choice \(1/(2K) = \max\left\{\frac{\ub-u}{1 - \lambda^\star(\ub-u)}, \frac{u}{1+\lambda^\star u} \right\}\) and the upper bound $\kappa = \re \kappa_2 = \re \kappa_1 \geq c_\kappa \cdot \frac{\alpha_0}{\ualpha}$
    with $c_\kappa = 1/2 \cdot \log\left(1 + \frac{1}{4}\left( \frac{\ub - \up f}{\ub}\right)^2\right)$ led to
    \[
        \bar r_2(v) \leq \frac{49\sqrt{6}}{36 \sqrt{\sigma^2}} \cdot \sqrt{\frac{\ualpha}{\alpha_0}} \frac{\ub}{\ub - \up f}\cdot
    \]

    Our next goal is to analyze the second term $\tilde r_2(v)$. We apply Taylor expansions of the form $T'(s) = T''(0) s + T'''(\xi_2) s^2/2$ and $P(s) = P(0) + P'(\eta) s$ to derive
    \begin{align*}
        \tilde r_2(v) &= \left\vert \frac{ P(s) - P(0) }{T'(s)} \right\vert =  \frac{\vert  P'(\eta) \cdot s \vert}{\vert T''(0)s + T'''(\xi_2) s^2/2 \vert} \leq \frac{ \sup_{\eta \in (0,s)} \vert  P'(\eta) \vert}{\vert T''(0) + T'''(\xi_2) s/2 \vert}\cdot
    \end{align*}
    First note that $P'(\eta)$ maximizes at $\eta = 0,$ since
    \[
        P'(\eta) = \frac{\ub - u}{(1 - \lambda^\star(\ub - u) + \rmi \eta (\ub - u))^2}\cdot
    \]
    Next by defining a random variable $Y_s = \frac{f(X) - u}{1 - \lambda^\star(f(X) - u) + is (f(X) - u)}$  and due to our choice of $K$ we conclude that 
    \[
        \vert T''(0) + T'''(\xi_2) s/2 \vert \geq \E[Y_0^2] - s \E[\vert Y_0\vert^3] \geq \E[Y_0^2]/2 = \sigma^2/2.
    \]
    It yields
    \[
        \tilde r_2(v) \leq  \frac{2(\ub - u)
        }{(1 - \lambda^\star(\ub - u))^2 \sigma^2} = \frac{2}{(1 - \lambda^\star(\ub - u))\sqrt{\sigma^2}} \sqrt{\frac{\frac{(\ub - u)^2}{(1 - \lambda^\star(\ub - u))^2}}{\E[Y_0^2] }}\cdot
    \]
    By a bound
    \[
        \E[Y_0^2] = \sum_{i=0}^{m} \frac{\alpha_i}{\ualpha} \cdot \left(\frac{f(i) - u}{1 - \lambda^\star(f(i) - u)}\right)^2 \geq \frac{\alpha_0}{\ualpha} \frac{(\ub - u)^2}{(1 - \lambda^\star(\ub - u))^2}
    \]
    we obtain
    \[
        \tilde r_2(v) \leq \frac{2}{(1 - \lambda^\star(\ub - u))\sqrt{\sigma^2}} \sqrt{\frac{\ualpha}{\alpha_0}}
    \]
    and
    \[
        \vert r_2(v) \vert \leq \frac{1}{(1 - \lambda^\star(\ub - u))\sqrt{\sigma^2}} \sqrt{\frac{\ualpha}{\alpha_0}} \left(2 + \frac{49 \sqrt{6}}{36} \frac{\ub}{\ub - \up f} \right).
    \]
    A similar bound  also holds for $r_1(v)$ by symmetry. Finally, due to bound on $\kappa$ and $\alpha_0 \geq 2,$ we derive
   \begin{align*}
        \vert R_2(\alpha) \vert &\leq \frac{2}{(1 - \lambda^\star(\ub - u))\sqrt{\sigma^2}} \sqrt{\frac{\ualpha}{\alpha_0}} \left(2 + \frac{49 \sqrt{6}}{36} \frac{\ub}{\ub - \up f} \right) \cdot \left\vert \int_0^{\kappa_2} e^{-\ualpha v} \rmd v + \int_0^{\kappa_1} e^{-\ualpha v} \rmd v \right\vert\\
        &\leq \frac{1}{(1 - \lambda^\star(\ub - u))\sqrt{\sigma^2}} \left(8 + \frac{49 \sqrt{6}}{9} \frac{\ub}{\ub - \up f} \right) \cdot \frac{1}{\sqrt{\ualpha \cdot \alpha_0}}\cdot
   \end{align*}

    \paragraph{Term $R_1$.}
    The analysis of this term can be carried out as in \cite{tiapkin22dirichlet} except the multiplication with $P(0)$,
    \[
        \vert R_1(\alpha) \vert \leq \frac{c_1}{\sqrt{ \sigma^2 c_\kappa \alpha_0} \cdot (1 - \lambda^\star(\ub - u))} \cdot \frac{\exp(-c_\kappa \alpha_0) }{\ualpha^{1/2}}\CommaBin
    \]
    where $c_1 = 2 \sqrt{2}$.
    
    \paragraph{Term $R_3$.}
    We start from the bound
    \begin{align*} 
\label{eq: R3-int}
\left\vert \int_{K}^\infty P(s) \exp(-\ualpha T(s))\, \rmd s \right\vert &\leq \exp(-\ualpha \cdot \re [T(K) - T(0)] ) \cdot \exp(-\ualpha T(0)) \\
&\cdot \sup_{s \in \R }\vert P(s) \vert  \int_{K}^\infty \exp(-\ualpha \re [T(s) - T(K)])\, \rmd s.
\end{align*}
Let us start from the analysis of an additional multiplier connected to $P(s)$
\[
    \sup_s \vert P(s) \vert = \sup_{s} \sqrt{\frac{1}{(1 - \lambda^\star(\ub - u))^2 + s^2 (\ub - u) }} = \frac{1}{1 - \lambda^\star(\ub - u)}\cdot
\]
The rest of the analysis coincides the the analysis of the same term in \cite{tiapkin22dirichlet}
\[
    \vert R_3(\alpha) \vert \leq c_3\cdot \frac{\exp(-\ualpha \Kinf(\up, u, f))}{1 - \lambda^\star(\ub - u)} \cdot \frac{1 - \lambda^\star(\ub-u)}{\ub-u}  \exp(-c_\kappa \alpha_0)
\]
for $c_3 = \sqrt{5\pi}/2$.
\end{proof}

Finally, we use a bounds on remainder terms to derive a lower bound on the density.

\begin{lemma}\label{lem:lb_dirichlet_density}
     Consider a function $f \in \Fclass_m(\ub,b)$ and a vector $\alpha = (\alpha_0 + 1, \alpha_1, \ldots, \alpha_m) \in \R_+^{m+1}$ with $\ualpha  \geq 2\alpha_0$, $\ub \geq 2b$. Let $w \sim \Dir(\alpha)$ and assume that $Z = wf$ is non-degenerate. Let $\varepsilon \in (0,1)$. Assume
     \[
        \alpha_0 \geq \left(\frac{4}{\sqrt{\log(17/16)}} + 8 + \frac{49 \cdot 4 \sqrt{6}}{9} \right)^2 \frac{2}{\pi \cdot \varepsilon^2} + \log_{17/16}\left( \frac{5}{32 \cdot \varepsilon^2}\right) + \log_{17/16}(\ualpha).
    \]
    Then for any  $u \in (\up f, \ub)$, 
    \[
        p_Z(u) \geq (1 - \varepsilon) \sqrt{\frac{\ualpha}{2\pi}}\frac{\exp(-\ualpha \Kinf(\up, u, f))}{(1 - \lambda^\star(\ub - u)) \sqrt{\sigma^2}}\cdot
    \]
\end{lemma}
\begin{proof}
    We start the proof from the combination of Proposition~\ref{prop:dirichlet_density_integral} and Proposition~\ref{prop:asymptotics_integral}
    \[
        p_Z(u) \geq \frac{\ualpha}{2\pi} \left( \left(\sqrt{2\pi} - \frac{1}{\sqrt{\alpha_0}}\left( \frac{c_1}{\sqrt{c_{\kappa}}} + c_2 \right) \right) \frac{\exp(-\ualpha \Kinf(\up, u, f))}{(1 - \lambda^\star(\ub - u)) \sqrt{\ualpha \cdot \sigma^2}} - \vert R_3(\alpha)\vert  \right).
    \]
    Since $\ualpha \geq 2\alpha_0$ and $\ub \geq 2b$ we have $\ub/(\ub - \up f) \leq 4$. In this case we have $c_\kappa \geq 1/2 \log(17/16)$ and $c_2 \leq 8 + 49\sqrt{6} \cdot 4 / 9$. Therefore
    \[
         \frac{c_1}{\sqrt{c_{\kappa}}} + c_2 \leq \frac{4}{\sqrt{\log(17/16)}} + 8 + \frac{49 \cdot 4 \sqrt{6}}{9} \triangleq \gamma_1.
    \]
    And for $\alpha_0 \geq 4\gamma_1^2/(2\pi \cdot \varepsilon^2),$ we have
    \begin{align*}
        p_Z(u) &\geq \frac{\ualpha}{2\pi} \left( \sqrt{2\pi}(1 - \varepsilon/2) \frac{\exp(-\ualpha \Kinf(\up, u, f)}{(1 - \lambda^\star(\ub - u)) \sqrt{\ualpha \cdot \sigma^2}} - \vert R_3(\alpha)\vert  \right)  \\
        &\geq \frac{\sqrt{\ualpha}}{2\pi} \left( \frac{\sqrt{2\pi}(1 - \varepsilon/2)}{\sqrt{\ualpha \, \sigma^2}} - c_3 \cdot \frac{1 - \lambda^\star(\ub - u)}{\ub - u} \cdot \exp(-c_{\kappa} \alpha_0) \right) \frac{\exp(-\ualpha \Kinf(\up, u, f)}{(1 - \lambda^\star(\ub - u))}\cdot
    \end{align*}
    
    Note that $\E[Y_0] = 0$ and observe that the inequality 
    \begin{align*}
        \sigma^2 &= \E[Y_0^2] = \Var[Y_0] \leq \left(\frac{\ub-u}{2(1 - \lambda^\star(\ub-u))} + \frac{u}{2(1 + \lambda^\star u)}\right)^2 \\
        &= \frac{\ub^2}{4(1-\lambda^\star(\ub - u))^2(1 + \lambda^\star u)^2} \leq \frac{4(\ub - u)^2}{(1 - \lambda^\star(\ub - u))^2}\CommaBin
    \end{align*}
    follows from a general bound on variance of bounded random variables (bounded differences), the fact (see Lemma~12 in \cite{honda2010asymptotically})
    \[
        \lambda^\star \geq \frac{u - \up f}{u(\ub - u)} \iff 1 + \lambda^\star u \geq \frac{\ub - \up f}{\ub - u}\CommaBin
    \]
    and the inequality $\ub/(\ub - \up f) \leq 4$. It yields
    \begin{align*}
        p_Z(u) \geq \frac{\sqrt{\ualpha}}{2\pi} \left( \sqrt{2\pi}(1 - \varepsilon/2) - 2 \sqrt{5\pi} \exp(-c_{\kappa} \alpha_0) \cdot \sqrt{\ualpha} \right) \frac{\exp(-\ualpha \Kinf(\up, u, f))}{(1 - \lambda^\star(\ub - u)) \cdot \sqrt{\sigma^2}}\cdot
    \end{align*}
    To guarantee 
    \[
         2 \sqrt{5\pi} \exp(-c_{\kappa} \alpha_0) \cdot \sqrt{\ualpha} \leq \sqrt{2\pi} \cdot (\varepsilon/2)
    \]
    we have to choose
    \[
        \alpha_0 \geq \log_{17/16}(5/(32 \varepsilon^2)) + \log_{17/16}(\ualpha).
    \]
    It allows us to conclude
    \[
        p_Z(u) \geq (1 - \varepsilon) \sqrt{\frac{\ualpha}{2\pi}}\frac{\exp(-\ualpha \Kinf(\up, u, f))}{(1 - \lambda^\star(\ub - u)) \cdot \sqrt{\sigma^2}}\cdot
    \]
\end{proof}

Before proceeding with the final proof, we derive one important technical result. 
\begin{lemma}\label{lem:kinf_lower_bound}
    For any $u \in (\up f, \ub)$ it holds 
    \[ 
    \Kinf(\up, u, f) \geq \frac{1}{2} (\lambda^\star)^2 \sigma^2  \big(1-\lambda^\star(\ub-u)\big)^2\,. 
    \]
\end{lemma}

\begin{proof}
Define the function $\phi_u(\lambda) = \E \log\big(1-\lambda(f(X)-u)\big)$ and $\lambda_u = \lambda^\star$. Remark that $\sigma^2 = -\phi_u''(\lambda_u)$. Thanks to the Taylor expansion of $\phi_u$ and the definition of $\lambda_u$ it holds 
\begin{align*}
    0 = \phi_u(0) &= \phi_u(\lambda_u) + 0 + \frac{\lambda_u^2}{2} \phi_u''(y\lambda_u)
\end{align*}
for some $y \in (0,1)$. Thus we can rewrite $\Kinf$ as 
\[
\phi_u(\lambda) = \frac{\lambda_u^2}{2} (-\phi_u''(y\lambda_u))\,.
\]
We will lower bound the opposite of the second derivative that appears above. First note that 
\[
-\phi_u''(y\lambda_u) = \E\left[ \frac{(f(X)-u)^2}{\big(1-\lambda_u(f(X)-u)\big)^2}\left(\frac{1-\lambda_u(f(X)-u)}{1-y\lambda_u(f(X)-u)}\right)^2\right]\,.
\]
We now lower-bound the ratio, noting that if $X\leq u$ then since $y\in(0,1)$
\[
\frac{1-\lambda_u(f(X)-u)}{1-y\lambda_u(f(X)-u)} \geq 1\,.
\]
In the other case $X>u$, we have $0\leq 1-y\lambda_u(f(X)-u) \leq 1$ and $1-\lambda_u(f(X)-u) \geq 1-\lambda_u(\ub-u)$ thus
\[
\frac{1-\lambda_u(f(X)-u)}{1-y\lambda_u(f(X)-u)} \geq 1-\lambda_u(b-u) > 0\,.
\]
In both case using $1-\lambda_u(\ub-u) \leq 1$ we get 
\[
\left(\frac{1-\lambda_u(f(X)-u)}{1-y\lambda_u(f(X)-u)}\right)^2 \geq \big(1-\lambda_u(\ub-u)\big)^2\,.
\]
In particular, using the definition of $\phi''(u)$, it entails that 
\[-\phi_u''(y\lambda_u) \geq -\phi_u''(\lambda_u) \big(1-\lambda_u(b-u)\big)^2\,.\]
Plugging this inequality in the integral representation of $\phi_u$ allows us to conclude
\[
\phi_u(\lambda) \geq \frac{1}{2} \lambda_u^2 \big(-\phi''(\lambda_u)\big)  \big(1-\lambda_u(b-u)\big)^2\,.
\]
\end{proof}

Using this lemma we may proceed with the proof of our final result.
\begin{proof}[Proof of Theorem~\ref{thm:lower_bound_dbc}]
    Define $Z = wf$. By Lemma~\ref{lem:lb_dirichlet_density},
   \begin{align*}
        \PP{wf \geq \mu} &= \int_{\mu}^{\ub} p_{Z}(u) \rmd u \geq (1 - \varepsilon) \sqrt{\frac{\ualpha}{2\pi}}  \cdot  \int_{\mu}^{\ub}  \frac{ \exp(-\ualpha \Kinf(\up, u, f))}{\sqrt{\sigma^2 (1 - \lambda^\star(\ub - u))^2}} \, \rmd u.
    \end{align*}
    By Theorem 6 by \cite{honda2010asymptotically},
    \[
        \frac{\partial}{\partial u}\Kinf(\up, u, f) = \lambda^\star.
    \]
    Thus, we can define a change of variables $t^2/2 = \Kinf(\up, u, f), t \rmd t = \lambda^\star \rmd u$ and write
    \begin{align*}
        \PP{Z \geq \mu} = (1 - \varepsilon)\int_{\sqrt{2 \Kinf(\up, \mu, f)}}^{+\infty} D(u) \sqrt{\frac{\ualpha}{2\pi}} \exp(-\ualpha t^2/2) \rmd t,
    \end{align*}
    where $D(u)$ is defined as a positive square root  of 
    \[
        D^2(u) = \frac{2 \Kinf(\up, u, f)}{(\lambda^\star)^2 \sigma^2 (1 - \lambda^\star(\ub - u))^2}.
    \]
    By Lemma~\ref{lem:kinf_lower_bound}, $D^2(u) \geq 1$ and hence
    \begin{align*}
        \PP{Z \geq \mu} &\geq (1 - \varepsilon)\int_{\sqrt{2 \Kinf(\up, \mu, f)}}^{+\infty} \sqrt{\frac{\ualpha}{2\pi}} \exp(-\ualpha t^2/2) \rmd t \\
        &= (1 - \varepsilon) \P_{g \sim \cN(0,1)}\big( g \geq \sqrt{2 \ualpha \Kinf(\up, \mu, f)} \big).
    \end{align*}
\end{proof}
\newpage
\section{Technical lemmas}
\label{app:technical}
\subsection{On the Bernstein inequality}
\label{app:Bernstein}
In this part, we restate Bernstein-type inequality of~\citet{talebi2018variance}.
\begin{lemma}[Corollary 11 by \citealp{talebi2018variance}]\label{lem:Bernstein_via_kl}
Let $p,q\in\simplex_{S-1},$ where $\simplex_{S-1}$ denotes the probability simplex of dimension $S-1$. For all functions $f:\ \cS\mapsto[0,b]$ defined on $\cS$,
\begin{align*}
	p f - q f &\leq  \sqrt{2\Var_{q}(f)\KL(p,q)}+\frac{2}{3} b \KL(p,q)\\
  q f- p f &\leq  \sqrt{2\Var_{q}(f)\KL(p,q)}\,.
\end{align*}
where use the expectation operator defined as $pf \triangleq \E_{s\sim p} f(s)$ and the variance operator defined as
$\Var_p(f) \triangleq \E_{s\sim p} \big(f(s)-\E_{s'\sim p}f(s')\big)^2 = p(f-pf)^2.$
\end{lemma}

\begin{lemma}[Lemma E.3 by \citealp{tiapkin22dirichlet}]
\label{lem:switch_variance_bis}
Let $p,q\in\simplex_{S-1}$ and a function $f:\ \cS\mapsto[0,b]$, then
\begin{align*}
  \Var_q(f) &\leq 2\Var_p(f) +4b^2 \KL(p,q)\,,\\
  \Var_p(f) &\leq 2\Var_q(f) +4b^2 \KL(p,q).
\end{align*}
\end{lemma}

\begin{lemma}[Lemma E.4 by \citealp{tiapkin22dirichlet}]
	\label{lem:switch_variance}
	For $p,q\in\simplex_{S-1}$, for $f,g:\cS\mapsto [0,b]$ two functions defined on $\cS$, we have that
	\begin{align*}
 \Var_p(f) &\leq 2 \Var_p(g) +2 b p|f-g|\quad\text{and} \\
 \Var_q(f) &\leq \Var_p(f) +3b^2\|p-q\|_1,
\end{align*}
where we denote the absolute operator by $|f|(s)= |f(s)|$ for all $s\in\cS$.
\end{lemma}

\newpage
\section{Lazy version of \texorpdfstring{\OPSRL}{OPSRL}}
\label{app:lazy_OPSRL}

In this section we present \lazyOPSRL a lazy version of the \OPSRL algorithm. Following \citet{efroni2019tight}, instead of computing new Q-values by backward induction before each episode in \lazyOPSRL we just just do one step of optimistic incremental planning at the current state to obtain improved Q-values (at the current state) and act greedily with respect to them. Precisely, given initial optimistic value functions $\uV_h^{-1}(s) = \ur H$ for all $(h,s)\in[H]\times\cS'$ and Q-function $\uQ_h^{-1}(s,a) = \ur H$ for all $(h,s,a) \in [H] \times \cS' \times \cA$ we update Q-values by applying the Bellman operator \emph{only at the visited states}:
\begin{equation}\label{eq:lazy_OPSRL_update_rule}
    \begin{split}
        \uQ_h^t(s,a) &\triangleq \ind\{s=s_h^{t+1}\} \left(r_h(s,a)+\max_{j\in[J]}\{ \tp_h^{\,t,j} \uV_{h+1}^{t-1}(s,a)\}\right)+ (1-\ind\{s=s_h^{t+1}\}) \uQ_h^{t-1}(s,a)  \,,\\
        \uV_h^t(s) &\triangleq \min\left\{\max_{a\in\cA} \uQ_h^t(s, a), \uV_h^{t-1}(s) \right\}\,,\\
        \pi_h^{t+1}(s) &\in \argmax_{a\in\cA} \uQ_h^t(s,a)\,,
    \end{split}
\end{equation}
where the posterior sample are still given by $\tp_h^{\,t,j}(s,a)\sim \Dir\!\Big(\big(\upn_h^t(s'|s,a)/\kappa\big)_{s'\in\cS'}\Big)$ and $\uV_{H+1}^t(s) = 0$ for all $t,s$. Consequently \lazyOPSRL enjoys a better time complexity of $\tcO(HSA)$ per episode than the one $\tcO(HS^2A)$ of \OPSRL.

The complete description of \lazyOPSRL is given in Algorithm~\ref{alg:lazyOPSRL} for a general family of probability distribution parameterized by the pseudo-counts over the transitions instead of the Dirichlet inflated prior/posterior.

\begin{algorithm}[h!]
\centering
\caption{\lazyOPSRL}
\label{alg:lazyOPSRL}
\begin{algorithmic}[1]
  \STATE {\bfseries Input:} Family of probability distributions $\rho: \N_{+}^{S+1} \to \Delta_{\cS'}$ over transitions, initial pseudo-count $\upn_h^0$, number of posterior samples $J$, initial value functions $\uV_h^{\,-1}$, initial Q-functions $\uQ_h^{\,-1}$.
      \FOR{$t \in[T]$}
      \FOR{$h \in [H]$}
        \STATE Sample $J$ independent transitions $\tp_h^{\,t-1,j}(s,a)\sim \rho\big(\upn_h^{t-1}(s'|s,a)_{s'\in\cS'}\big),\quad j\in[J]$.
        \STATE Compute for all $a\in\cA$
        \begin{align*}
            \uQ_h^{t-1}(s_h^t,a) &= r_h(s_h^t,a)+\max_{j\in[J]} \bigl\{\tp_h^{\,t-1,j} \uV_{h+1}^{t-2}(s_h^t,a)\bigr\}\,,\\
            \uV_h^{t-1}(s_h^t) &= \min\left\{\max_{a\in\cA} \uQ_h^{t-1}(s^t_h, a), \uV_h^{t-2}(s^t_h) \right\}\,.
         \end{align*}
        \STATE Play $a_h^t \in \argmax_{a\in\cA} \uQ_h^{t-1}(s_h^t,a)$.
        \STATE Observe $s_{h+1}^t\sim p_h(s_h^t,a_h^t)$.
        \STATE Increment the pseudo-count $\upn_h^t(s_{h+1}^t | s^t_h, a_h^t)$. 
      \ENDFOR
   \ENDFOR
\end{algorithmic}
\end{algorithm}

Interestingly, we can also obtain a regret bound for \lazyOPSRL of the same order as \OPSRL with the same number of posterior samples as in \ref{eq:def_J}.
\begin{theorem}
\label{th:regret_bound_lazyOPSRL} 
Consider a parameter $\delta \in (0,1)$. Let $\kappa \triangleq 2(\log(12 SAH/\delta) + 3\log(\rme\pi(2T+1)))$, $n_0 \triangleq \lceil \kappa(c_{0} + \log_{17/16}(T)) \rceil$, $\ur \triangleq 2$, where  $c_{0}$ is an absolute constant defined in \eqref{eq:constant_c0}; see Appendix~\ref{app:optimism}. Then for \lazyOPSRL, with probability at least $1-\delta$, 
\[
    \regret^T = \cO\left( \sqrt{H^3 SAT L^3}  + H^3 S^2 A L^3 \right),
\]
 where $L \triangleq \cO(\log(HSAT/\delta))$.
\end{theorem}

\begin{proof}
Since this proof is very similar to the one of Theorem~\ref{th:regret_bound_OPSRL} we only describe how it needs to be adapted.

\paragraph{Optimism} We are going to show that on event $\cE^{\anticonc}(\delta)$ (see Proposition~\ref{prop:anticonc} for definition) our estimate of Q-function is optimistic that is $\uQ^t_h(s,a) \geq \Qstar_h(s,a)$ for any $(t,h,s,a) \in \{0,\ldots,T\} \times [H] \times \cS \times \cA$ and $\uV^t_h(s) \geq \Vstar_h(s)$ for $(t,h,s) \in \{-1,\ldots,T\} \times [H] \times \cS$.

We prove by forward induction on $t$ and backward induction on $h$. Base of induction $t=-1$ and $h=H+1$ is trivial. Next, if $s \not = s^{t+1}_h$ then $\uQ^t_h(s,a) = \uQ^{t-1}_h(s,a)$ and the statement is correct by the induction hypothesis. In the case of $s=s^{t+1}_h$ we have by Bellman equations and update rule \eqref{eq:lazy_OPSRL_update_rule}
\[
    \uQ^t_h(s,a) - \Qstar_h(s,a) = \max_{j\in[J]}\{ \tp_h^{\,t,j} \uV_{h+1}^{t-1}(s,a)\} - p_h\Vstar_{h+1}(s,a).
\]
By induction on $t$ and $h$ we have $\uV_{h+1}^{t-1}(s') \geq \Vstar_{h+1}(s')$ for any $s' \in \cS$ thus by combination with event $\cE^{\anticonc}(\delta)$ we conclude the statement.

\paragraph{Regret bound} 

    Recall $\delta^t_h = \uV^{t-1}_h(s^t_h) - V^{\pi^t}_h(s^t_h)$ and $\uregret^T_h = \sum_{t=1}^t \delta^t_h$. By update rule for value function $\uV^t_h(s^t_h) \leq \uQ^t_h(s^t_h, a^t_h)$. Thus we can proceed by update rule for Q-function and Bellman equations
    \begin{align*}
        \delta^t_h &\leq \uQ^{t-1}_h(s^t_h, a^t_h) - Q^{\pi^t}_h(s^t_h, a^t_h) =  \max_{j \in [J]}\left\{ \tp^{\,t-1,j}_h \uV^{t-2}_{h+1}(s^t_h, a^t_h) \right\} - p_h V^{\pi^t}_{h+1}(s^t_h, a^t_h) \\
        &= \underbrace{\max_{j \in [J]}\left\{ \tp^{\,t-1,j}_h \uV^{t-2}_{h+1}(s^t_h, a^t_h) \right\} - \up^{\,t-1}_h \uV^{t-2}_{h+1}(s^t_h, a^t_h)}_{\termA} + \underbrace{[\up^{\,t-1}_h - \hp^{\,t-1}_h] \uV^{t-2}_{h+1}(s^t_h, a^t_h)}_{\termB} \\
        &+ \underbrace{[\hp^{\,t-1}_h - p_h] [\uV^{t-2}_{h+1} - \Vstar_{h+1}] (s^t_h, a^t_h)}_{\termC} +  \underbrace{[\hp^{\,t-1}_h - p_h] \Vstar_{h+1}(s^t_h, a^t_h)}_{\termD} \\
        & + \underbrace{p_h[\uV^{t-2}_{h+1} - V^{\pi^t}_{h+1}](s^t_h, a^t_h) - [\uV^{t-2}_{h+1} - V^{\pi^t}_{h+1}](s^t_{h+1})}_{\xi^t_h} + \underbrace{[\uV^{t-2}_{h+1} - \uV^{t-1}_{h+1}](s^t_{h+1})}_{\Delta^t_h} + \delta^t_h.
    \end{align*}
    
    Here we see that all terms are very similar to the terms that appears in the proof of Lemma~\ref{lem:surrogate_regret_bound} except the additional one $\Delta^t_h \triangleq [\uV^{t-2}_{h+1} - \uV^{t-1}_{h+1}](s^t_{h+1})$. By adapting the concentration event $\cG^{\conc}(\delta)$ with a shift of indices we may obtain the following upper bound (for $N^t_h > 0$)
    \begin{align*}
        \delta^t_h &\leq \left(1 + \frac{1}{H} \right) \delta^t_h + \left(1 + \frac{1}{H} \right) \Delta^t_h + \left(1 + \frac{1}{H} \right) \xi^t_h \\
        & + 3L\sqrt{\frac{\Var_{\up^{\,t-1}_h}[\uV^{t-2}_{h+1}](s^t_h,a^t_h)}{\upN^{\,t}_h}} + \sqrt{2L} \cdot \sqrt{\frac{\Var_{p_h}[\Vstar_{h+1}](s^t_h,a^t_h)}{N_h^t}}\\
        &+ \frac{10 H^2S \cdot L}{N^{\,t}_h} + \frac{16 L^2 H}{\upN^t_h}.\\
    \end{align*}
    Thus, the surrogate regret is bounded by almost the same quantity up to a shift of indices and one additional term
    \begin{align*}
        \uregret^{\,T}_h \leq \tilde{A}^T_h + B^T_h + C^T_h + 4\rme H\sqrt{2 H T L} + 2\rme H^2 SA + \sum_{t=1}^T \sum_{h'=h}^H \gamma_{h'}\Delta^t_{h'},
    \end{align*}
    where $\gamma_{h} = (1 + 1/H)^{H-h+1}$ and
    \begin{align*}
        \tilde{A}^T_{h} &= 3\rme L \sum_{t=1}^T \sum_{h'=h}^H \sqrt{\Var_{\up^{\,t-1}_{h'}}[\uV^{t-2}_{h+1}](s^t_{h'},a^t_{h'}) \cdot \frac{\ind\{ N^{t}_{h'} > 0\}}{N^{t}_{h'}}}\CommaBin \\
        B^T_{h} &= \rme \sqrt{2L}\sum_{t=1}^T \sum_{h'=h}^H\sqrt{\Var_{p_{h'}}[\Vstar_{h+1}](s^t_{h'},a^t_{h'}) \frac{\ind\{ N^{t}_{h'} > 0\}}{N_{h'}^t}}\CommaBin  \\
        C^T_{h} &= 26 H^2 S L^2 \sum_{t=1}^T \sum_{h'=h}^H \frac{\ind\{ N^{t}_{h'} > 0\}}{N^t_{h'}}.
    \end{align*}
    
    The terms $B^T_h$ and $C^T_h$ remain exactly the same as in the analysis of \OPSRL, whereas there will be a small difference in the analysis $\tilde{A}^T_h$. 
    
    Next, we analyze the new term using non-increasing of the value function $\uV^{t-1}_h(s) \leq \uV^{t-2}_h(s)$
    \begin{align*}
        \sum_{t=1}^T \sum_{h'=h}^H \gamma_{h'} \Delta^t_{h'} \leq \rme \sum_{t=1}^T \sum_{h'=h}^H \Delta^t_{h'}\,.
    \end{align*}
    We derive a bound on the sum of $\Delta^t_h$ over $T$ for any fixed $h$ by a telescoping property
    \begin{equation}\label{eq:bound_sum_delta}
        \begin{split}
        \sum_{t=1}^T \Delta^t_h &= \sum_{s \in \cS} \sum_{t=1}^T \ind\{s = s^t_{h+1} \} [\uV^{t-2}_{h+1} - \uV^{t-1}_{h+1}](s) \\
        &\leq \sum_{s \in \cS} \sum_{t=1}^T [\uV^{t-2}_{h+1} - \uV^{t-1}_{h+1}](s) = \sum_{s \in \cS}[\uV^{\,-1}_{h+1} - \uV^{T-1}_{h+1}](s) \leq 2 H S.
        \end{split}
    \end{equation}
    Thus we have a next bound for surrogate regret
    \begin{align*}
        \uregret^{\,T}_h \leq \tilde{U}^T_h \triangleq  \tilde{A}^T_h + B^T_h + C^T_h + 4\rme H\sqrt{2 H T L} + 4\rme H^2 SA.
    \end{align*}
    
    Next we explain the analysis of term $\tilde{A}^T_1$. To do it, we analyze the sum of variance by following the step of Lemma~\ref{lem:sum_variance}. All analysis remain exactly the same except the analysis of term $\termFour$, that can be handled by additional use of inequality \eqref{eq:bound_sum_delta}
    \begin{align*}
        \termFour &= \sum_{t=1}^T \sum_{h=1}^H \ur H p_h (\uV^{t-2}_{h+1} - V^{\pi^t}_{h+1})(s^t_h, a^t_h) \\
        &= 2H \sum_{t=1}^T \sum_{h=1}^H (\xi^t_h + \delta^t_h + \Delta^t_h) \leq 4 H^2 \sqrt{2TL} + 2 H^2 \tU^T_1 + 2H^2 S.
    \end{align*}
    The only difference is in the term $2 H^2 S$ that is a second-order term. Thus, the following version of Lemma~\ref{lem:sum_variance} holds for \lazyOPSRL
    \begin{align*}
        \sum_{t=1}^T \sum_{h=1}^H  \Var_{\up^{\,t-1}_h}[\uV^{t-1}_{h+1}](s^t_h,a^t_h) \ind\{N^t_h > 0\} &\leq 2 H^2 T + 2 H^2 \tilde{U}_1^T  + 40 H^3S^2A L^3 + 32 H^3 S\sqrt{2AT L}
    \end{align*}
    with the change only in a constant in front of the third term. The rest of the proof remains the same as in the analysis of \OPSRL.
\end{proof}
\newpage

\section{Experimental details}
\label{app:experiments}

In this appendix we provides details on comparison \OPSRL with some baselines and additionally study the impact of choice of the number of posterior samples $J$ for \PSRL and the impact of optimistic prior for \OPSRL and \PSRL. Our code is published on \href{https://github.com/d-tiapkin/optimistic-psrl-experiments}{GitHub} and based on the library \texttt{rlberry} by \cite{rlberry}.

\paragraph{Environment} We use a grid-world environment with $100$ states $(i, j) \in [10]\times[10]$ and $4$ actions (left, right, up and down). The horizon is set to $H=50$. When taking an action, the agent moves in the corresponding direction with probability $1-\epsilon$, and moves to a neighbor state at random with probability $\epsilon$. The agent starts at position $(1, 1)$. The reward equals to $1$ at the state $(10, 10)$ and is zero elsewhere. 
\paragraph{Number of posterior samples} First we investigate the influence of the number of posterior samples $J$ on the regret. We fixed the other parameters as follows: We use the prior over the transition probability described in Section~\ref{sec:algorithm} with $n_0 = 1$ initial pseudo-counts and no inflation $\kappa = 1$. In Figure~\ref{fig:opsrl_samples} we plot the regret of \OPSRL in the environment described above when the number of posterior samples varies in $J \in\{1,4,8,16,32\}$. We observe that the number of posterior samples has little effect on the regret, especially if we compare it to the scale of the gap between the different regret curves of the baselines in Figure~\ref{fig:regret_baselines}. Thus, in the sequel of this appendix, we arbitrarily choose $J=8$. Another justification of this choice is that $J \approx \log(T)$ for $T = 10000$, as it was required by theoretical analysis.

\begin{figure}[h!]
      \vspace{-0.2cm}
    \centering
    \includegraphics[keepaspectratio,width=.75\textwidth]{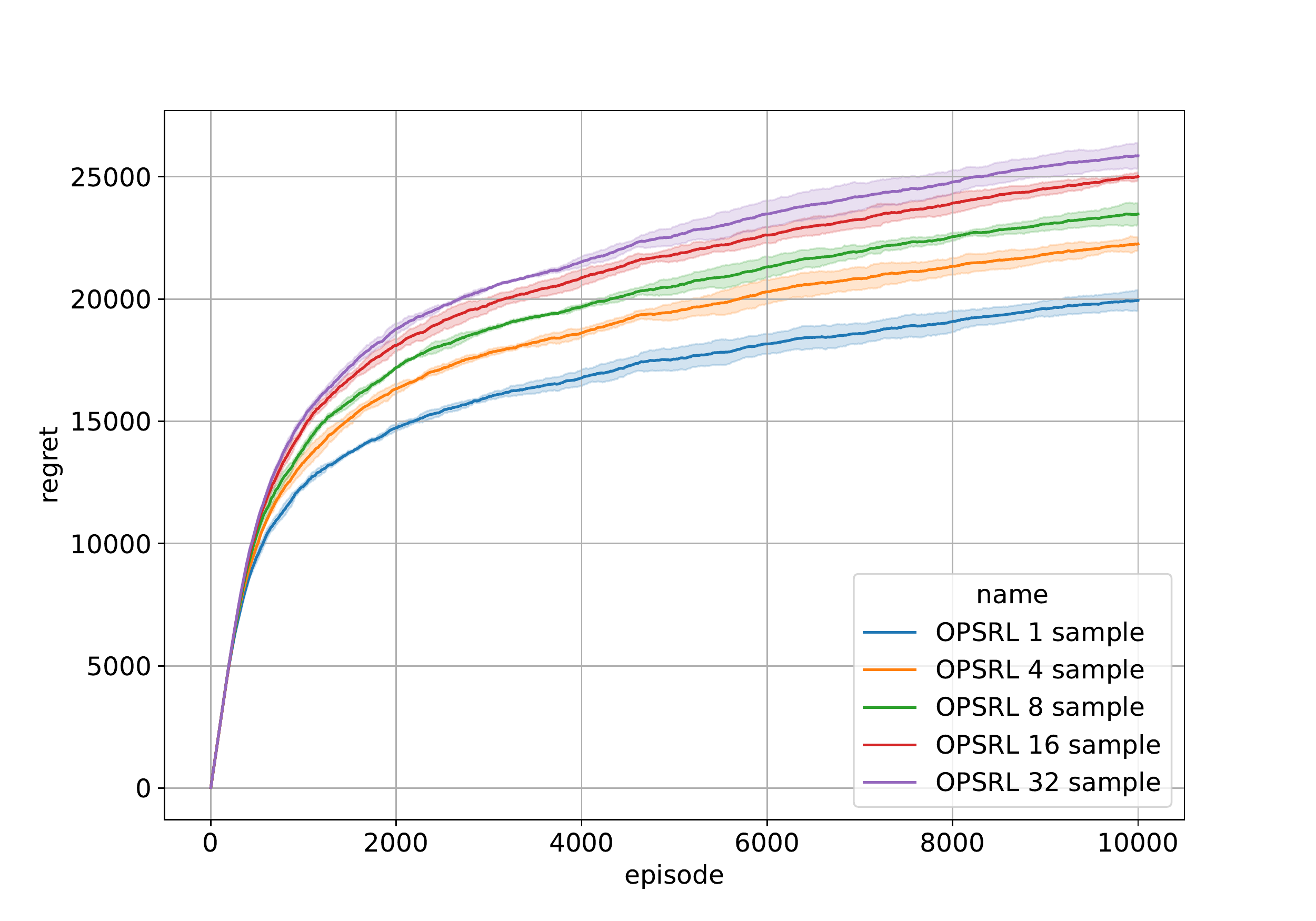}
    \caption{Regret of \OPSRL for $J\in\{1,4,8,16,32\}$ for $H=50$ and transitions noise $0.2$. We show $\text{average}$ over $4$ seeds.}
    \label{fig:opsrl_samples}
\end{figure}

\paragraph{Baselines} We compare \OPSRL with the following baselines: 
\begin{itemize}
    \item The \UCBVI algorithm by \citep{azar2017minimax} (with Hoeffding-type bonuses). Since the theoretical bonus often leads to poor practical performance we use simplified bonuses from an idealized Hoeffding inequality of the form 
    \begin{align*}
	\beta_h^t(s,a) =
	\min\left(
	\sqrt{\frac{(H-h+1)^2}{4 n_h^t(s,a)}}, H-h+1
	\right)\,.
    \end{align*}
    \item The \UCBVIB algorithm, the same algorithm as above but with simplified bonuses from an idealized Bernstein inequality:
    \begin{align*}
	\beta_h^t(s,a) =
	\min\left(
	\sqrt{\frac{\Var_{\hp^{t}}[V_{h+1}^{t-1}](s,a)}{n_h^t(s,a)}} + \frac{H-h+1}{n_h^t(s,a)}, H-h+1
	\right)\,.
    \end{align*}
    \item The \PSRL algorithm by  \citep{obsband2013more}. For this algorithm we used a Dirichlet distribution of parameter $(1/S,\ldots,1/S)$ as prior on the transition probability.
    \item The \RLSVI algorithm by \citep{osband16generalization}. As for \UCBVI we used a simplified variance for the Gaussian noise 
    \begin{align*}
	\sigma_h^t(s,a) =
	\min\left(
	\sqrt{\frac{(H-h+1)^2}{4 n_h^t(s,a)}}, H-h+1
	\right)\,.
    \end{align*}
\end{itemize}
For the \OPSRL we use the prior over the transition probability described in Section~\ref{alg:OPSRL} with $n_0 = 1$ initial pseudo-counts and no inflation $\kappa = 1$. Note that the number of pseudo-counts is the same that for the one of the chosen prior for \PSRL (where the sum of parameters is also one).
As discussed above we pick $J=8$ posterior samples.


\paragraph{Results} In Figure~\ref{fig:regret_baselines}, we plot the regret of the various baselines and \OPSRL in the grid world environment. In this experiment, we observe that \OPSRL achieves competitive results with respect to \PSRL. It is not completely surprising since they share the same Bayesian model on the transitions up to the prior. We shall elaborate more on the influence of the prior below. We also note that \OPSRL outperforms \UCBVI and \RLSVI.
This difference may be explained by the fact that \OPSRL's optimism implies (in the worst case) KL bonuses as in \citet{filippi2010optimism}. The KL bonuses are stronger than Bernstein bonuses, see Lemma~\ref{lem:Bernstein_via_kl}, because they somehow rely on all moments of the empirical distribution rather than the first two moments as in the case of Bernstein bonuses or first moments for Hoeffding bonuses or for the variance of the Gaussian noise in \RLSVI. Note also that in \OPSRL, we do not have to solve the complex convex program to compute the KL bonuses \citet{filippi2010optimism}, which could be computationally intensive. 

\paragraph{Influence of prior} Next we study the influence of the prior for posterior sampling algorithms. Here we will refer to \OPSRL as an optimistic prior choice and to \PSRL as a uniform prior choice.

\begin{figure}[h!]
      \vspace{-0.2cm}
    \centering
    \includegraphics[keepaspectratio,width=.75\textwidth]{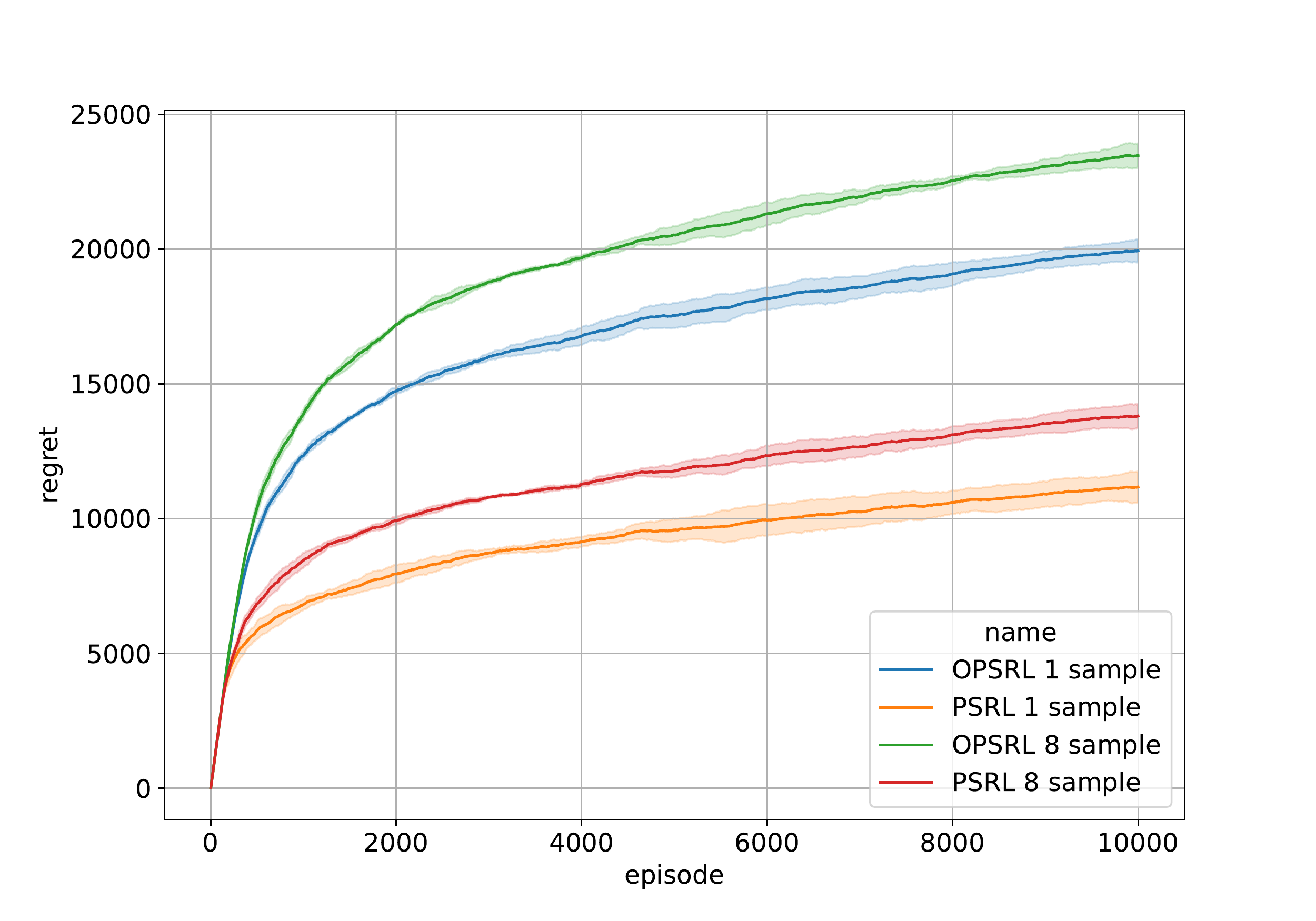}
    \caption{Regret of \OPSRL with optimistic prior and \PSRL with uniform prior for $J \in \{1,8\}$ for $H=50$ an transitions noise $0.2$. We show $\text{average}$ over $4$ seeds.}
    \label{fig:prior_influence}
\end{figure}

On Figure~\ref{fig:prior_influence} we may observe that algorithm convergences for both tested numbers of Thompson samples $J$ and the only difference is the speed of forgetting the prior distribution that results in a constant difference between regrets.  We see that optimistic prior is slightly harder to forget and it is connected to one of the most interesting features of it: optimistic prior is robust  to the choice of the underlying probabilistic model. This property makes it universal at the price of efficiency on particular examples.

\end{document}